\documentclass{article}

\usepackage{arxiv}

\usepackage{amsmath,amsfonts}
\usepackage{amsthm}
\usepackage{graphicx}
\usepackage[utf8]{inputenc}
\usepackage{calrsfs}
\usepackage{enumitem}
\usepackage{tikz}
\newcommand*\circled[1]{\tikz[baseline=(char.base)]{
            \node[shape=circle,draw,inner sep=2pt] (char) {#1};}}
\usepackage{algorithm} 
\usepackage{algpseudocode}
\usepackage{url} 
\usepackage{mathtools}

\def\modelset{{\mathfrak S _ \mathcal H }} 
 
\def\x{{\mathbf x}}
\def\s{{\mathbf s}}
\def\z{{\mathbf z}}
\def\y{{\mathbf y}}

\def\A{{\mathbf A}}
\def\Q{{\mathbf Q}}
\def\M{{\mathbf M}}

\def\R{{\mathbb R}}
\def\E{{\mathbb E}}

\DeclareMathOperator{\vc}{vec}
\DeclareMathOperator*{\argmin}{arg\,min}
\usepackage{xcolor}

\theoremstyle{definition}
\newtheorem{definition}{Definition}[section]
\newtheorem{corollary}{Corollary}[section]
\newtheorem{lemma}{Lemma}[section]
\newtheorem{remark}{Remark}

\newtheorem{theorem}{Theorem}[section]

\DeclareMathAlphabet{\pazocal}{OMS}{zplm}{m}{n}

\title{Compressive Independent Component Analysis: Theory and Algorithms}

\author{{
 Michael P. Sheehan and Mike E. Davies}\\
Institute of Digital Communications\\
University of Edinburgh\\
Edinburgh, UK\\
Corresponding author: michael.sheehan@ed.ac.uk}

\renewcommand{\shorttitle}{Compressive Independent Component Analysis}

\begin{document}
\maketitle

\begin{abstract}
{Compressive learning forms the exciting intersection between compressed sensing and statistical learning where one exploits forms of sparsity and structure to reduce the memory and/or computational complexity of the learning task. In this paper, we look at the independent component analysis (ICA) model through the compressive learning lens. In particular, we show that solutions to the cumulant based ICA model have particular structure that induces a low dimensional model set that resides in the cumulant tensor space. By showing a restricted isometry property holds for random cumulants  e.g. Gaussian ensembles, we prove the existence of a compressive ICA scheme. Thereafter, we propose two algorithms of the form of an iterative projection gradient (IPG) and an alternating steepest descent (ASD) algorithm for compressive ICA, where the order of compression asserted from the restricted isometry property is realised through empirical results. We provide analysis of the CICA algorithms including the effects of finite samples. The effects of compression are characterised by a trade-off between the sketch size and the statistical efficiency of the ICA estimates. By considering synthetic and real datasets, we show the substantial memory gains achieved over well-known ICA algorithms by using one of the proposed CICA algorithms. Finally, we conclude the paper with open problems including interesting challenges from the emerging field of compressive learning.}

\end{abstract}

\section{Introduction}
In recent years, the size of datasets have grown exponentially as a result of advances in technology, signal acquisition, and the sophistication of modern day mobile phones and devices. This has enabled researchers, statisticians and machine learning practitioners to build increasingly accurate models as a consequence of larger sample sizes and feature dimensions. Nevertheless, this poses a fundamental challenge to large scale learning as (i) traditional algorithms have computational complexity that scales with the order of the dataset dimensions (ii) the whole dataset has to be stored or transferred on to local RAM as optimisation methods need to return to the data (or a random subset of the data) at subsequent iterations, and (iii) one is vulnerable to malicious attacks of potentially sensitive and personal information as the data needs to be stored or transferred locally. Compressive learning (CL) \cite{gribonval2021compressive,gribonval2021sketching} partially addresses these fundamental challenges by severely compressing the whole dataset into a random representation of fixed  size, named a so-called sketch, in a single (or limited) pass of the data prior to learning. Once the sketch is formed, the parameters of the model are inferred solely from the sketch, hence a CL algorithm, for a given task or model, needs never to return to the original dataset, and it can be deleted from memory as a result. At the core of the CL framework \cite{gribonval2021compressive,keriven2018sketching}, is that in general, the size of the sketch does not scale with the dimensions of the dataset, or indeed the data's underlying dimensionality, but instead is driven by the complexity or dimensionality of the task or model of interest. In theory, one can work with datasets of arbitrary length, as the dimension of the sketch is fixed constant throughout, making CL especially amenable to large scale learning. Inferring the parameters of a model solely from the sketch is an under determined inverse problem. As a result, we need regularity assumptions to make the problem well-posed. These assumptions come in the form of a low dimensional model set that the solution to the inference problem lies on or close to. The reader may notice this is reminiscent of compressive sensing where one assumes the signal of interest is $k$ sparse in some domain, and therefore the solution lies on or close to the union of $k$ dimensional subspaces representing a low dimensional model set. The sparse regularity assumption allows one to take a limited number of measurements to recover the signal of interest and reduce the complexity and cost of acquisition. In later sections, we take inspiration from compressive sensing to develop and analyse our CL algorithms.
\par In this paper, we develop a CL framework, including theory and practical algorithms, for independent component analysis (ICA). ICA is an unsupervised learning task that attempts to find the linear transformation that separates some given data into components of maximal independence. It is used extensively in the machine learning and signal processing communities for example as a dimensionality reduction tool \cite{hyvarinen2000independent}, to uncover underlying factors that effect the price movements of a collection of stocks \cite{882456} and to detect independent sources in the brain through EEG signals \cite{841330}. As will be discussed in section \ref{sec: background}, the ICA problem can be solved directly from the data or through some higher order statistics of the data, such as the kurtosis. Given that the number of independent sources is denoted by $n$ and the signal or data length is denoted by $N$, then the memory complexity typically scales either with $\mathcal{O}(nN+n^2)$ or $\mathcal{O}(n^4)$ depending on the method of choice. As one can see, this becomes infeasible for large scale datasets. In this paper, we show theoretically and empirically that it is possible to design a CL ICA algorithm where the sketch dimension, and therefore the memory complexity, scales at $\mathcal{O}(n^2)$ which can be orders of magnitudes smaller than current approaches.

\subsection{Contributions and Outline}
Below, we highlight the main contributions of the paper:
\begin{itemize}
    \item Focusing on the higher order statistics method of ICA, we show that an independent component model set exists in the space of $4^{\text{th}}$ order cumulant tensors and we state the set's dimensionality.
    \item We prove that the well-known restricted isometry property (RIP) holds for projections of the cumulant tensor on the model set with a sketch size  $m=\mathcal{O}(n^2)$, which is order optimal to the dimensional of the model set. Given a sketch constructed from random Gaussian ensembles, we show that the reconstruction error is stable when measuring a signal which lives close to, but not on, the ICA model set through the existence of an instance optimal decoder.
   \item Two inherently different CICA algorithms are proposed, in the form of an alternating steepest descent and an iterative projection gradient algorithm and show that they successfully recover the ICA mixing matrix with overwhelming probability provided that the sketch size $m\geq 2n(n+1)$. We analyse the proposed CICA algorithms on synthetic and real datasets showing that the CICA scheme achieves substantial memory savings over existing ICA methods, whilst retaining competitive estimation accuracy.
   \item We analyse the tradeoff between the compression rate of the sketch and the overall statistical efficiency of the sketch estimate. 
\end{itemize}
\subsection{Related Works}
\subsubsection{Existing Compressive Learning Models}
The framework of CL has been successfully applied to a host of learning tasks and models with the desired outcome of reducing the complexities associated with signal acquisition, computation and memory storage. In \cite{keriven2018sketching}, Keriven \textit{et al.} proposed a CL framework for mixture models, in particular the mixture of Gaussian distributions and $k$-means learning tasks. In both cases, a sketch is constructed by randomly sampling the characteristic function of the mixture model which can be equivalently seen as taking random Fourier features of the data \cite{rahimi2008random}. The compact representational sketch of each mixture model scales as $\mathcal{O}(k^2d)$, where $k$ is the number of mixtures in the model and $d$ is the feature space dimensions of the data. As a result, a compressive mixture model algorithm was proposed that had both computational and space complexities that scaled independently of the number of data points $N$. In \cite{gribonval2021compressive}, a compressive principal component analysis (PCA) framework was proposed. As will be discussed in Section \ref{subsec: Compressive PCA}, the compressive PCA methodology is aligned closely to our compressive ICA framework. Distinct from compressive mixture models, the compressive PCA method is left distribution free and it is assumed that the data lives on, or can be approximately modelled, by a $k$-dimensional subspace. As a result, a sketch of size $\mathcal{O}(kd)$ can be computed by taking a random projection of the covariance matrix $\Sigma\in\R^{d\times d}$ of the data, hence reducing the memory complexities of storing either the data of size $Nd$ or the covariance matrix of size $d^2$. 
\subsubsection{Generalised Method of Moments}
Compressive learning is similar to the technique of Generalised Method of Moments (GeMM) \cite{hansenGeMM,gemmhall} where the parameters of interest $\theta$ are estimated by matching a collection of generalised moments of the distribution with the empirical counterparts calculated through the data. In most cases it is used instead of maximum likelihood estimation when calculating the likelihood is not tractable. CL differs from much of the GeMM literature as the goal is fundamentally different: in compressive learning one attempts to construct a compact representation of data with the aim of reducing complexity constraints (computation, memory, acquisition) whilst in GeMM the goal is to primarily estimate $\theta$ when the model is either partially specified or the likelihood does not have a closed form solution. Moreover, the selected generalised moments may be a function of the parameter being estimated, hence not providing a one off sketch. 
\subsubsection{Streaming Methods}
Closely related to CL is the collection of streaming methods \cite{cormode2012synopses,tropp2019streamingScientific}, where data items are seen and queried only once by the user and then discarded. This is of particular interest when the summary statistic of choice is updated and maintained in real time, for example in the online learning setting \cite{1198387}, to reduce space complexities. Notably, the count-min-sketch \cite{cormodea2005improved} was developed to query data in an online fashion with the application of maintaining histograms of quantiles. However, these methods in general focus on the discrete collection of objects and database queries while in CL the framework and method is applied to machine learning tasks where typically the signal is question is continuous. Tropp \textit{et al.} \cite{tropp2017practical} proposed a streaming framework for large scale PCA. In particular, in \cite{tropp2019streamingScientific}, the authors design random sketches for on-the-fly compression of data matrices associated with large scale scientific simulations. Here the data matrix $\mathbf{A}$ of interest can be decomposed into a sequence
\begin{equation}
    \mathbf{A}=\mathbf{H}_1+ \mathbf{H}_2+\mathbf{H}_3+\dots 
\end{equation}
where it is assumed each $\mathbf{H}_i$ has some structural redundancies for example sparsity or low-rank. These methods have a subtle yet fundamental difference from CL, as in CL the structural assumptions which are exploited to form the CL sketch arise from the model or distribution itself, while in these streaming methods the structural assumptions come directly from the data. Moreover, several passes of the data may be required to reduce the low-rank approximation error \cite{tropp2017practical}. 
\subsubsection{Other Compression Techniques}
Coresets are a popular method used to compress a database into a summary statistic used for inferring the parameters of a given model and has been used primarily for subspace clustering based tasks \cite{har2004coresets,feldman2020turning}. In a similar vein to CL, the compact data representation has size that typically scales independently to the number of input items and the dimension of input feature space. However, the coresets are constructed in a hierarchical manner, possibly resulting in multiple passes of the data and are therefore not naturally amenable to online or distributed learning. Projections that include both random projections and feature selections \cite{boutsidis2010random,calderbank2009compressed} are used widely to reduce the dimensionality of the data. In \cite{calderbank2009compressed}, datasets were randomly projected into a compressed domain using both random Gaussian and Bernoulli matrices. In a similar vein to compressive sensing \cite{Donoho_CS}, the data was assumed to be $k$-sparse therefore the dependency of the feature space dimension $d$ was removed within the space and acquisition complexities. In contrast to random projections, more structural based projections are proposed. In \cite{tang2014feature}, different feature selection techniques for classification are reviewed including structured graph methods and the use of embedded models. The well-known PCA method is a popular preprocessing technique that projects the dataset onto a $k$-dimensional subspace of maximal variance \cite{jolliffe2016principal}. In both random and structured projections, the methods discussed only tackle the dependency of the feature space dimension $d$ and do not address the challenges posed by a large data size $N$. Sub-sampling methods are also a popular method for dimensionality reduction whereby a subset of the original dataset is used for learning. As discussed previously, the method of coresets \cite{har2004coresets,feldman2020turning} is a sub-sampling technique that attempts to sub-select dominant items that well approximate the structure of the dataset. Other sub-sampling techniques include random and adaptive sub-sampling \cite{cormode2012synopses}. The disadvantage of sub-sampling techniques is that there is a risk of discarding important information relating to non-sampled data items. Moreover, these techniques only tackle the constraint on the number of data items $N$ and don't combat the complexity issues posed by the feature space dimensional $d$. 

\par Specifically to ICA compression, Sela \textit{et al.} \cite{sela2016randomized} used kernel approximation techniques to reduce the dimensions of the Kernel ICA method proposed by Bach \cite{bach2002kernel}. Random Fourier features are used to approximate the kernel, reducing the memory complexity from $\mathcal{O}(d^2N^2)$ to $\mathcal{O}(MN)$, where $M$ is the number of random Fourier weights used. Despite the reduction in memory complexity, the algorithm still has storage demands which scale linearly with $N$. In comparison, we remove the dependency of the data length $N$ completely, within our framework, when estimating the ICA mixing matrix.

\section{Background}\label{sec: background}
\subsection{Compressive Learning}\label{subsec: Compressive learning}
Let $\mathbf{x}_1,\mathbf{x}_2,\dots,\mathbf{x}_N$ be independent and identically distributed samples from an unknown probability distribution $\pi$ on $(\pazocal{X},\pazocal{B})$ where $\pazocal{X} \subset\R^d$ is some Euclidean space and $\pazocal{B}$ is a Borel $\sigma$-field. Classically, $\pi$ is parametrized by some parameters denoted by $\theta\in\Theta (\in\R^k)$. A statistical learning problem can be formalised as follows: find a hypothesis $h^*$ from a hypothesis class $\mathcal{H}$ that best matches the probability distribution $\pi$ over the training collection $\{\mathbf{x}_i\}_{i=1}^N$, given some data fidelity term. Given a loss function $l:\pazocal{X}\times\mathcal{H}\longmapsto\R$, this is equivalent to minimizing the risk defined as
\begin{equation}
\label{Eqn: risk intro}
    h^* = \argmin_{h\in\mathcal{H}}\mathcal{R}(\pi,h)=\argmin_{h\in\mathcal{H}}\E_{\mathbf{x}\sim\pi}l(\mathbf{x},h).
\end{equation}
\noindent
Formally, the model set associated to the hypothesis class can be defined as:
\noindent
\begin{equation}
\label{eqn: generic model set }
    \mathfrak{S}_\mathcal{H}:=\{\pi\in\mathcal{P}(\pazocal{X}):\exists h \in \mathcal{H},\, \mathcal{R}(\pi,h)=0\}.
\end{equation}
\noindent In other words, the set containing all distributions for which zero risk is achievable. As a result, the model set has a dimension which is intrinsic to the hypothesis class of the model. In practice, one cannot minimize the true risk as we generally do not have access to the true distribution $\pi$, so instead, one can minimize the empirical risk with respect to the finite samples of the true distribution and as a result this may mean all the data is required to be stored in memory. 
\par In CL \cite{keriven2018sketching,gribonval2021compressive,gribonval2021sketching}, we find a compact representation, or a so-called sketch, that encodes some statistical properties of the data. Its size is ideally chosen relative to the intrinsic complexity of the problem, making it possible to work with arbitrarily large datasets while storing in memory an object of fixed size. Given a feature function $\Phi:\pazocal{X}\longmapsto\mathbb{C}^m$, such that $\Phi$ is integrable with respect to any $\pi\in\mathcal{P}(\pazocal{X})$, define a linear operator $\mathcal{A}:\mathcal{P}(\pazocal{X})\longmapsto\R^m$ by
\noindent
\begin{equation}
\label{Eqn: sketch}
    \mathcal{A}(\pi):=\E_{\mathbf{x}\sim\pi}\Phi(\mathbf{x}).
\end{equation}
\noindent
The sketch defined in (\ref{Eqn: sketch}) can be seen as taking the expectation of some particular features of the distribution $\pi$, which is similar to the field of kernel mean embedding \cite{muandet2017kernel} where one uses feature maps to embed probability distributions. Therefore, we would like to construct $\mathcal{A}$ so that $\mathcal{A}(\pi)$ captures sufficiently relevant information of the data to allow us to infer the parameters of the model directly from the sketch. As a trivial example, if we seek to infer only the mean of a normal distribution $\pi=\mathcal{N}(\mu,\sigma)$, the construction $\mathcal{A}(\pi)$ where $\Phi(\mathbf{x})=\mathbf{x}$ would constitute a trivial yet sufficient sketch. In reality, CL is applicable to much more complex models where the feature function is non-trivial and the model may not necessarily possess a finite dimensional sufficient statistic independent from the data. The goal of CL is to therefore construct a sketch of size  $m\ll Nd$ that captures enough information to recover an estimated risk which is \textit{close} to the true risk with high probability \cite{gribonval2021compressive}. In practice, as in the kernel mean embedding literature \cite{muandet2017kernel}, the empirical distribution is used to form an empirical sketch defined as 
\noindent
\begin{equation}
\label{eqn: empirical sketch}
    \hat{\mathbf{y}}=\mathcal{A}(\pi_N) \quad \text{where } \quad \pi_N:=\frac{1}{N}\sum_{i=1}^N\delta_{\mathbf{x}_i}
\end{equation}
\noindent
denoting by $\delta_{x}$ the Dirac distribution on $x$, and therefore the empirical sketch can be formed directly from the data. Due to the law of large numbers,  $\lim_{N\rightarrow\infty}\mathcal{A}(\pi_N)=\mathcal{A}(\pi)$. Once the sketch has been computed, one can discard the dataset $\{\mathbf{x}_i\}_{i=1}^N$ from memory. As a result, CL reduces down to solving an inverse problem of the form 
\begin{equation}
    \label{eqn: linear inverse prob background}
    \hat{\theta}=\argmin_{\theta\in\Theta} C(\theta\mid \hat{\y})
\end{equation}
where $C(\cdot\mid\hat{\y})$ is a cost function designed for the specific learning task at hand. In a compressive sensing light, we can exploit structural assumptions of the model set and the associated parameter space $\Theta$, e.g sparsity, low rankness, low dimensional manifold properties, to make (\ref{eqn: linear inverse prob background}) well-posed and finding a solution tractable. As such, one can design a decoder $\Delta$ that exploits the structural assumptions of the model set $\mathfrak{S}_\mathcal{H}$ to recover the parameters of the model from the sketch whilst minimizing the risk. The sketching operator $\mathcal{A}$ and the decoder $\Delta$ form the pair $(\Delta,\mathcal{A})$ that define the CL algorithm for a specific learning problem. It should be noted that minimizing (\ref{eqn: linear inverse prob background}) plays the role of a proxy for minimizing  the empirical risk as, by definition of the model set in (\ref{eqn: generic model set }), any $\pi\in\mathfrak{S}_\mathcal{H}$ has zero loss in expectation \cite{gribonval2021compressive}. 

\begin{table}
\centering
\resizebox{\columnwidth}{!}{\begin{tabular}{|l|l|l|l|}
\hline
\textbf{Learning Task }                              & $k$ - \textbf{Means }  & \textbf{GMM} & \textbf{PCA}  \\
\hline
Model set $\mathfrak{S}_h$  & $\left\{\pi\mid \text{mix. of } k \text{ Diracs}\right\}$ & $\left\{\pi\mid \text{mix. of } k \text{ Gaussians}\right\}$     & $\left\{\pi\mid\text{rank}(\Sigma_\pi)\leq k\right\}$     \\
Feature func. $\Phi(\x)$ &   $\left(e^{{\rm i}\omega_j^T\x}/w(\omega_j)\right)_{j=1}^m$ &        $\left(e^{{\rm i}\omega_j^T\x}\right)_{j=1}^m$      &     $(\langle\mathbf{A}_j,\x\x^T\rangle)_{j=1}^m$    \\
Sketch cost $C(\theta,\y)$ &  $\min_{\pi\in\mathfrak{S}_\mathcal{H}}\lVert \y-\mathcal{A}(\pi)\rVert_2$         & $\min_{\pi\in\mathfrak{S}_\mathcal{H}}\lVert \y-\mathcal{A}(\pi)\rVert_2$    & $\min\lVert\Sigma_\pi\rVert_*$ s.t $\mathcal{A}(\Sigma_\pi)=\y$      \\
Sketch Size  $m$                               &   $\mathcal{O}(k^2d)$              &  $\mathcal{O}(k^2d)$   & $\mathcal{O}(kd)$  \\
\hline
\end{tabular}}
\caption{Summary of existing methods in the CL framework. For more details see \cite{gribonval2021compressive}. }
\label{Table: CL summary}
\end{table}

\subsection{Compressive Principal Component Analysis}\label{subsec: Compressive PCA}
In Section \ref{subsec: Compressive learning}, the framework of CL was discussed in a general manner without specific consideration of the distributional form of the model. As will be discussed in Section \ref{subsec: ICA}, the PCA and ICA models are similar in nature in that the model is often left distribution free. In other words, the distribution of the sampled data is left unspecified. In Table \ref{Table: CL summary}, it is shown that the compressive PCA model set \cite{gribonval2021compressive}  is defined as
\begin{equation}
    \mathfrak{S}_\mathcal{H}= \left\{\pi\mid\text{rank}(\Sigma_\pi)\leq k\right\}.
\end{equation}
Due to the distribution free assumption of the PCA model, we seek structural assumptions that are manifested within some \textit{intermediary} statistic space $\mathbb{S}$ to make computing a sketch possible \cite{sheehan2019compressive}. In the case of compressive PCA, the space of $d\times d$ covariance matrices is leveraged as an intermediary statistic space $\mathbb{S}$ where the rank of the covariance matrices is exploited. Figure \ref{Fig: CL diagram} depicts a geometric viewpoint of both compressive parametric learning (e.g. $k$-means, GMM) and distribution free compressive learning (e.g. PCA, ICA). In general, distribution free CL poses distinct challenges and advantages from the typical parametric CL framework \cite{sheehan2019compressiveSPM}. Challenges arise when choosing an intermediary statistic space $\mathbb{S}$, for instance (1) what set of intermediate statistics can we use? (2) How do the structural assumptions of the model set manifest within the intermediate statistic? Equivalently, there are many advantages. Specifically, by leveraging some set of intermediate statistics we have implicitly mapped the problem from an infinite dimensional probability space to a typically finite dimensional statistic space. As a result, we can utilise a host of existing techniques within the compressive sensing literature to design encoder and decoder pairs $(\mathcal{A},\Delta)$. Moreover, it also allows us to use a more flexible semi-parametric model that is only partially specified. As will be discussed in Section \ref{sec: CICA Theory}, the compressive ICA framework follows a similar convention where the space of 4th order cumulant tensors $\mathbb{S}=\mathfrak{C}$ is used as an intermediary statistic space to exploit structural assumptions of the model set $\mathfrak{S}_\mathcal{H}$ to form a sketch.

\begin{figure}[ht!]
\centering
\includegraphics[width=5.5in]{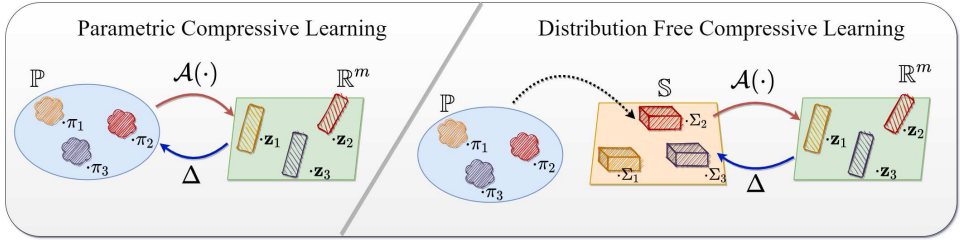}
\caption{A schematic diagram of parametric compressive learning (Top) and distribution free compressive learning (Bottom).  }
\label{Fig: CL diagram}
\end{figure}

\subsection{Independent Component Analysis}\label{subsec: ICA}
 ICA is used frequently in the machine learning  and signal processing communities to identify latent variables that are mutually independent to one another. Consider a data vector $\x=(x_1,x_2,\dots,x_d)^T$, then the problem of ICA concerns finding a mixing matrix $\M\in\R^{d\times n}$ (here we assume that $d\geq n$) such that
 \begin{equation}
     \label{eqn: ICA vector set up}
     \x=\M\s,
 \end{equation}
 where $\s=(s_1,s_2,\dots,s_n)^T$ and the components $s_i$ are statistically independent:
\begin{align}
\label{eqn: statistical independence}
    p(s_1,s_2,\dots,s_n)=\prod^n_{i=1}p_i(s_i).
\end{align}
The data point $\x$ is only one realisation of a data matrix or signal $\mathbf{X}\in\R^{N\times d}$ of length $N$, so therefore we attempt to infer $\M$ with the collective set of linear equations $\mathbf{X}=\M\mathbf{S}$, where $\s$ is a realisation of $\mathbf{S}\in\R^{N\times n}$. There are many techniques and methods in the literature to solve the ICA problem. The simplest method is to assume the distributional form of each of the independent components $p_i(s_i)$ and then solve the ICA problem through a maximum likelihood approach \cite{BellICA}. In practice, the distributions are not known a-priori  so therefore in most methods the distributions are left unspecified. As a result, practitioners and researchers often resort to minimizing a given contrast function to solve the ICA problem.

\subsubsection{Prewhitening}
A popular preprocessing trick for ICA is to prewhiten the data beforehand. This involves the process of finding the matrix $\mathbf{P}$ such that 
\begin{equation}
\label{eqn: prewhiten matrix}
    \z=\mathbf{P}^{-1}\x,
\end{equation}
where $\z$ has identity covariance matrix. This initial preprocessing step, which can be executed through singular value decomposition based techniques, uncorrelates the mixed components and handles the issue of when there are more mixing components than independent components $d>n$. Moreover, it has the advantage that the matrix $\mathbf{Q}=\mathbf{P}^{-1}{\mathbf{M}}$ to be found is necessarily  orthogonal and square. For the sake of presentation, we will subsequently consider the whitened version of the data for the remainder of this section and the corresponding whitened ICA equation
\begin{equation}
    \label{eq: whitened ICA set up}
    \z=\mathbf{Q}\s.
\end{equation}
In Section \ref{subsec: unwhitened discussion}, we propose 2 equivalent sketching frameworks that can either incorporate prewhitened and unwhitened data.
 
 \subsubsection{Cumulant Based ICA}\label{subsubsec: cumulant based ICA}
Tensorial or cumulant based methods are a group of techniques used to solve the ICA problem and are of particular interest in this paper. Statistical properties of the data instance $\z$ can be described by its cumulants $\mathcal{Z}^K_{i_1i_2\dots i_K}$. In the multivariate setting, cumulants give rise to tensors, denoted $\mathcal{Z}^K$ for a cumulant tensor of order $K$. Assuming the data has zero mean, the first four cumulants are defined \cite{de1997signal} as 
\begin{equation}
\label{eqns: higher order statistics}
  \begin{aligned}
   \mathcal{Z}^1_i=&0\\
   \mathcal{Z}^2_{ij}=&\E[z_iz_j] \\
   \mathcal{Z}^3_{ijk}=&\E[z_iz_jz_k]\\
    \mathcal{Z}^4_{ijkl}=&\E[z_iz_jz_kz_l]-\E[z_iz_j]\E[z_kz_l]-\E[z_iz_k]\E[z_jz_l]-\E[z_iz_l]\E[z_jz_k]\\
\end{aligned}  
\end{equation}
\noindent where $\E$ is the expectation operator.  Given the model in (\ref{eq: whitened ICA set up}) equating $\z$ to $\s$, then the following multilinear property holds for their associated cumulant tensors:

\begin{equation}
\label{eqn: first multilinear prop}
    \mathcal{Z}^K=\mathcal{S}^K\times_1\Q\times_2 \Q\times_3\dots\times_K\Q,
\end{equation}
 \noindent where $\times_j$ represents the $j$-mode tensor-matrix product and $\mathcal{S}^K$ represents the $K^{th}$ order cumulant tensor of the independent source signals \cite{de1997signal}. In this paper we will only consider $4^{th}$ order cumulant tensors (e.g. $K=4$) and for the sake of simplified notation we shall drop the superscript in (\ref{eqn: first multilinear prop}) for the rest of the discussion. We denote by $\mathfrak{C}\subset\R^{d\times d\times d\times d}$ the space of 4th order cumulant tensors which account for the symmetry in (\ref{eqns: higher order statistics}), where each cumulant tensor $\mathcal{Z}\in\mathfrak{C}$ has a maximum of $\binom{n+3}{4}$ unique entries (degrees of freedom) \cite{COMONTensorDiag}. The diagonal entries $\mathcal{Z}_{ijkl}$ $(ijkl=iiii)$ are the auto-cumulants of $\z$, while the off-diagonal entries  $\mathcal{Z}_{ijkl}$ $(ijkl\neq iiii)$ are the cross-cumulants. If the variables $(z_1,z_2,\dots,z_n)$ are statistically independent then, as seen by (\ref{eqns: higher order statistics}), the cross-cumulants vanish to 0 resulting in a strictly diagonal cumulant tensor. In other words, independence implies diagonality. It is shown in \cite{comon2009tensor} that under mild conditions the converse is also true, i.e. diagonality implies independence. Cumulant based ICA can therefore be seen as finding a linear transformation $\mathbf{W} (=\Q^T)$ such that the resulting cumulant tensor  \begin{equation}
     \label{Eqn: diagonal cumulant tensor}
     \mathcal{S}:=\mathcal{Z}\times_1 \mathbf{W} \times_2\mathbf{W} \times_3 \mathbf{W} \times_4 \mathbf{W}
 \end{equation}
 is strictly diagonal. We can define the following ICA model set:
 \begin{equation}
    \label{eqn: model set}
    \mathfrak{S}_\mathcal{H}:=\{\pi\mid \mathcal{Z}=\mathcal{S}\times_1\Q\times_2\Q\times_3\Q\times_4\Q,\,\,\mathcal{S}\in\mathfrak{D},\,\Q^T\Q=I\},
\end{equation}
where $\mathfrak{D}\in\mathfrak{C}$ is the set of diagonal cumulant tensors, defined formally as

\begin{equation}
\label{eqn: set of diagonal tensors}
    \mathfrak{D}\coloneqq\left\{\mathcal{S}\mid\mathcal{S}_{ijkl}=0\,\, \forall ijkl\neq iiii \text{ and }\mathcal{S}_{iiii}\geq \epsilon_\mathcal{S}\right\},
\end{equation}

\noindent Here, we have the additional requirement\footnote{A standard requirement in ICA is that at maximum one diagonal cumulant $\mathcal{S}_{iiii}$ can be zero which arises from the ICA assumption that at maximum one source signal $s_i$ is Gaussian \cite{hyvarinen2000independent}. Here we have the slightly stronger assumption that all source signals are non-gaussian.} that each diagonal cumulant is greater than or equal to a small constant $\epsilon_\mathcal{S}> 0$. The expected cumulant tensor $\mathcal{Z}$ is typically not known owing to finite data length approximations and non-Gaussian additive noise \cite{de2000introduction2ICA} and so in general $\mathcal{Z}$ cannot be \textit{fully} diagonalized by a linear transform. As a result, contrast functions are used to approximately diagonalize $\mathcal{Z}$ and maximize the independence of the system.

 \subsubsection{Contrast Functions}\label{susubsec: contrast functions}
A contrast function $\Psi:\mathbb{P}\mapsto\mathbb{R}$ is a mapping from the space of distributions $\mathbb{P}$ to the real line and is a measure of the statistical independence between latent variables in a linear system \cite{comon1994independent,hyvarinen2000independent}. For a function to be a contrast function it must be permutation and scaling invariant, and also maximized if and only if the distributions are statistically independent \cite{comon1994independent}. For instance, the negative mutual information satisfies the contrast conditions, although it can be difficult to estimate in practice. Typically, contrast functions are tractable approximations of information theoretical measures such as negative mutual information, maximum likelihood and negentropy. Comon proposed various cumulant based contrast functions in \cite{comon1994independent}, that are Edgeworth expansions of information theoretic measures. The simplest is given by 
 \begin{equation}
     \label{eqn: cumulant contrast function}
     \Psi(\mathbf{W})=\sum^n_{i=1} \hat{\mathcal{S}}_{iiii}.
 \end{equation}
 where $\hat{\mathcal{S}}$ is the $4^{th}$ order cumulant tensor corresponding to the variable $\hat{\s}=\mathbf{W}\z$. When $\Psi(\mathbf{W})$ is maximum, the components of $\hat{\s}$ are independent giving $\hat{\s}=\s$ and $\mathbf{W}=\Q^T$. A 4th order cumulant tensor that has a decomposition as given in (\ref{Eqn: diagonal cumulant tensor}) and is therefore a member of the model set, $\mathcal{Z}\in\mathfrak{S}_\mathcal{H}$, maximizes any given cumulant based contrast function \cite{comon1994independent}, and hence minimizes the associated information theoretic measure and the risk in (\ref{Eqn: risk intro}).
 For further details, comprehensive reviews of cumulants and tensors can be found in \cite{comon1994independent,de1997signal}.
 
\par As discussed in Section \ref{subsec: Compressive PCA}, the 4th order cumulant tensor will act as an intermediary statistic space ($\mathbb{S}=\mathfrak{C}$). It is well documented through identifiability results in the cumulant based ICA literature \cite{comon1994independent,de2000introduction2ICA} that the parameters of the ICA model, namely the mixing matrix $\Q$, can be estimated solely from the 4th order cumulant tensor. As such, the 4th order cumulant tensor can be seen in its own right as a sketch, albeit inefficient with respect to compression being $\mathcal{O}(n^4)$. In the next section, we motivate the principles behind sketching the 4th order cumulant tensor to form a compact representational sketch that has size $\mathcal{O}(n^2)$.

\section{Compressive Learning Principles for Cumulant ICA}\label{sec: Motivation}
It was discussed in Section \ref{subsubsec: cumulant based ICA} that the model set $\mathfrak{S}_\mathcal{H}$ of the ICA problem defined in (\ref{eqn: model set}), maximizes any given cumulant based contrast function \cite{comon1994independent}. The model set $\mathfrak{S}_\mathcal{H}$ is itself a low-dimensional space residing in the space of cumulant tensors $\mathfrak{C}$. Specifically, $\mathfrak{S}_\mathcal{H}$ can be described as the product set of the set of $n\times n$ orthogonal matrices, denoted $O(n)$, and the set of diagonal cumulant tensors $\mathfrak{D}$ that was defined in (\ref{Eqn: diagonal cumulant tensor}). We can therefore initially count the degrees of freedom of the model set $\mathfrak{S}_\mathcal{H}$:

\begin{itemize}
    \item $\mathfrak{D}$ - A maximum of $n$ degrees of freedom on the leading diagonal.
    \item  $O(n)$ -  A maximum of  $\frac{n(n-1)}{2}$ degrees of freedom \cite{Szarek_OrthogonalCovering}.
\end{itemize}
In total, the model set has $\frac{n(n+1)}{2}$ degrees of freedom. In comparison, the space of 4th order cumulant tensors $\mathfrak{C}$, in which the model set resides, has $p\coloneqq\binom{n+3}{4}\approx\mathcal{O}(n^4)$ degrees of freedom. As the model set is of low complexity, in principle we could form a sketch of the 4th order cumulant tensor $\mathcal{Z}$ and estimate the parameters of the ICA model, namely the mixing matrix $\Q$, solely from the sketch. The sketch of the 4th order cumulant tensor $\mathcal{Z}$ is defined by
\begin{equation}
    \label{eqn: sketch motivation}
    \y^\textbf{w}=\mathcal{A}\left(\mathcal{Z}\right),
\end{equation}
where $\textbf{w}$ denotes that the sketch is acting on the whitened data $\z$. The computation of the sketch is very related to the sketching method of compressive PCA highlighted in Table \ref{Table: CL summary}. Akin to compressive PCA, the sketching operator $\mathcal{A}$ acts on the finite dimensional space of 4th order cumulant tensors instead of the infinite dimensional probability space which is left unspecified due to the nature of the ICA model. The ICA sketch defined in (\ref{eqn: sketch motivation}) draws strong connections to finite dimensional compressive sensing \cite{candes2008introduction,Donoho_CS} where limited (random) measurements of a finite dimensional sparse vector are taken to reduce the complexities associated with signal acquisition. Throughout the compressive sensing literature \cite{candes2008introduction,Donoho_CS,candes2011tight}, the restricted isometry property (RIP) is fundamental tool that is extensively used to show that a sketching operator $\mathcal{A}$ stably embeds elements of the model set into a compressive domain $\R^m$, provided that the sketch dimension $m$ is of sufficient size. In other words, given a sketching operator $\mathcal{A}$, it proves that the distance between every pair of signals in the model set are approximately preserved under the action of the sketch therefore providing a near isometry. In the case of compressive ICA, given two elements of the ICA model set $\mathcal{Z}_1,\mathcal{Z}_2\in\mathfrak{S}_\mathcal{H}$ and an RIP constant $\delta\in(0,1)$, then 
\begin{equation}
    \label{Eqn: Initial RIP}
    (1-\delta)\lVert\mathcal{Z}_1-\mathcal{Z}_2\rVert^2\leq \lVert\mathcal{A}(\mathcal{Z}_1-\mathcal{Z}_2)\rVert^2\leq (1+\delta)\lVert\mathcal{Z}_1-\mathcal{Z}_2\rVert^2
\end{equation}
provided that the sketch size $m$ is of sufficient dimension. In many cases, the sketch size $m$ is sufficient to be of the order of the degrees of freedom of the model set. In \cite{bourrier2014fundamental,BlumensathIPG11}, it is proved that if the lower RIP (LRIP) holds for a given sketching operator $\mathcal{A}$, e.g. the left of (\ref{Eqn: Initial RIP}), then there exists a robust decoder $\Delta$ that recovers a signal from the model set in a stable manner with respect to noise and signals that lie close to the model set. Moreover, it is proved in \cite{bourrier2014fundamental} that if the LRIP holds for the sketching operator $\mathcal{A}$ on the model set $\mathfrak{S}_\mathcal{H}$ then the decoder $\Delta$ can be the constrained $\ell_2$ optimization, for instance
\begin{equation}
    \label{eqn: initial constrained l_2}
    \Delta\left(\y^\textbf{w},\mathcal{A}\right) \in \min_{\mathcal{Z}\in\mathfrak{S}_\mathcal{H}}\lVert\y^\textbf{w}-\mathcal{A}(\mathcal{Z})\rVert_2.
\end{equation}
In principle, if the RIP can be proved for a sketching operator $\mathcal{A}$ on the ICA model set $\mathfrak{S}_\mathcal{H}$, then we have an optimization strategy for solving the compressive ICA problem.

\section{Compressive Independent Component Analysis Theory}\label{sec: CICA Theory}
We begin by explicitly defining the sketching operator $\mathcal{A}:\mathfrak{C}\mapsto\R^m$ as
\begin{equation}
\label{Eqn: Subgauss sketch operator}
    \mathcal{A}(\mathcal{Z})=\A \text{vec}(\mathcal{Z}),
\end{equation}
where $\A\in\R^{m\times p}$ and $\text{vec}$ denotes the vectorization operator. Here we assume $\A$ is some random measurement matrix where the entries $\A_{ij}$ are sampled according to some distributing law, $\A_{ij}\sim\Lambda$. In this paper, we consider two randomized linear dimension reduction maps, namely the Gaussian map and the subsampled randomized Hadamard transform (SRHT) stated below. The CICA RIP, our main result stated in Theorem \ref{Thm: RIP Result}, is proved using the Gaussian map, however fast Johnson-Lindenstrauss transforms (FJLT), for instance the SRHT, still work in practice as will be discussed in Section \ref{sec: results}.

\subsubsection{Gaussian Maps}
The most traditional randomized linear dimension reduction map is the subgaussian matrix which has been used extensively in the CS literature \cite{candes2008introduction,Donoho_CS}. The subgaussian matrix $\mathbf{A}\in\R^{m\times p}$ has entries that follow
\begin{equation}
    \label{eqn: subgaussian matrix entries sketch}
    \mathbf{A}_{ij}\sim\mathcal{N}\left(0,m^{-\frac{1}{2}}\right).
\end{equation}
Gaussian maps typically require $\mathcal{O}(mp)$ in memory as well as exhibiting a computational complexity of $\mathcal{O}\left(mp\right)$.
\subsubsection{Subsampled Randomized Hadamard Transform}\label{subsubsec: SRHT sec}
The SRHT is an instance of a FJLT that approximates the properties of the full Gaussian map \cite{krahmer2011new}. Here $\mathbf{A}\in\R^{m\times p}$ is defined as 
\begin{equation}
    \label{Eqn: SRHT Sketch}
    \mathbf{A}=\frac{1}{\sqrt{mp}}\mathbf{DH}\mathbf{S}_\text{sub},
\end{equation}
where
\begin{itemize}
    \item $\mathbf{D}\in\R^{p\times p}$ is a diagonal matrix whose elements are independent random signs $\{1,-1\}$.
    \item $\mathbf{H}\in\R^{p\times p}$ is a normalised Walsh-Hadamard matrix.
    \item $\mathbf{S}_\text{sub}\in\R^{m\times p}$ is a matrix consisting of a a subset of $m$ randomly sampled rows from the $p\times p$ identity matrix.
\end{itemize}
\noindent The SRHT is particularly cheaper to compute and store in comparison to the Gaussian map. As we do not explicitly store $\mathbf{H}$, the SRHT only requires $\mathcal{O}\left(m+p\right)$ in memory \cite{tropp2011improved}. In addition, the computational complexity of computing the sketch reduces to $\mathcal{O}\left(p\log(m)\right)$ in comparison to using the Gaussian map \cite{ailon2006approximate,tropp2011improved}. Below we state our main result of the paper.

\begin{theorem}[Compressive ICA RIP]
\label{Thm: RIP Result}
Let $\mathcal{Z}_1,\mathcal{Z}_2\in\modelset$ and denote $\mathcal{A}$ the Gaussian map sketching operator defined in (\ref{eqn: subgaussian matrix entries sketch}). Then $\mathcal{A}$ satisfies the RIP in (\ref{Eqn: Initial RIP}) with constant $\delta\in (0,1)$ and probability $1-\xi$ provided that
\begin{equation}
    m\geq\frac{C}{\delta^2}\max\left\{2n(n+1)\log(C_5),\log\left(\frac{6}{\xi}\right)\right\},
\end{equation}
where $C>0$ is an absolute constant and $C_5$ is a constant defined in Lemma \ref{lem: covering number of the normalized model set}.
\end{theorem}
\noindent The proof of Theorem \ref{Thm: RIP Result} is detailed in Section \ref{subsec: proof of main theorem}.
\begin{corollary}[Information Preservation]
\label{Corr: info preserve}
Let $\mathcal{Z}^*\in\mathfrak{C}$ be an arbitrary 4th order cumulant tensor and denote $\mathbf{y}^\textbf{w}=\mathcal{A}\left(\mathcal{Z}^*\right)+\mathbf{e}$ where $\mathbf{e}\in\R^m$ is some additive noise. Furthermore, let  $\Tilde{\mathcal{Z}}\coloneqq\Delta\left(\mathbf{y}^\textbf{w},\mathcal{A}\right)$ denote the solution to (\ref{eqn: initial constrained l_2}). Given that $\mathcal{A}$ satisfies the RIP in Theorem \ref{Thm: RIP Result}, then with probability $1-\xi$
\begin{align}
 \label{Eqn: Bound in main Result}
 \begin{split}
  \lVert\mathcal{Z}^*-\tilde{\mathcal{Z}}\rVert_F\leq& \min_{\mathcal{Z}\in\modelset}\Big(2\lVert\mathcal{Z}^*-\mathcal{Z}\rVert_F
  +\frac{2}{\sqrt{1-\delta}}\lVert\mathbf{A}\vc\left(\mathcal{Z}^*-\mathcal{Z}\right)\rVert_2\Big)+\frac{2}{\sqrt{1-\delta}}\lVert\mathbf{e}\rVert_2+\nu,      
 \end{split}
 \end{align}
 where $0<\nu\leq1$ is a small positive constant.
\end{corollary}
\begin{proof}
Given the LRIP in Theorem \ref{Thm: RIP Result}, we use Theorem 7 in \cite{bourrier2014fundamental} to obtain our result.
\end{proof}
\par The proof of Theorem \ref{Thm: RIP Result} uses covering numbers and $\epsilon$-nets of the normalized secant set of $\modelset$.

\begin{definition}[Secant Set]
The secant set of a set $\modelset$ is defined as 
\begin{equation}
    \modelset-\modelset:=\left\{\mathcal{Y}=\mathcal{Z}_1-\mathcal{Z}_2\mid\mathcal{Z}_1,\mathcal{Z}_2\in\modelset\right\}.
\end{equation}
\end{definition} 

\begin{definition}[Normalised Secant Set]
The normalized secant set $\mathfrak{N}\left(\modelset-\modelset\right)$ of a set $\modelset$  is defined as 
\begin{equation}
    \mathfrak{N}\left(\modelset\right):=\left\{\mathcal{Y}/\lVert\mathcal{Y}\rVert_F\mid\mathcal{Y}\in\left(\modelset -\modelset\right)\setminus\{\mathbf{0}\}\right\},
\end{equation}
where $\mathbf{0}$ defines the zero tensor. 
\end{definition}

\begin{definition}{(Covering number)} Let $\epsilon>0$. The covering number $\text{CN}(\modelset,\lVert\cdot\rVert,\epsilon)$ of a set $\modelset$ is the \textit{minimum number} of closed balls of radius $\epsilon$, with respect to the norm $\lVert\cdot\rVert$, with centres in $\modelset$ needed to cover $\modelset$. The set of centres of these balls is a minimal $\epsilon$-net for $\modelset$.
\end{definition} 

\begin{lemma}[Covering number of $\mathfrak{N}\left(\modelset-\modelset\right)$]
\label{lem: covering number of the normalized model set}
The covering number of $\mathfrak{N}\left(\modelset-\modelset\right)$ with respect to the Frobenius norm $\lVert \cdot\rVert_F$ is 
\begin{equation}
    \text{CN}\left(\mathfrak{N}\left(\modelset-\modelset\right),\lVert\cdot\rVert_F,\epsilon\right)\leq\left(\frac{C_5}{\epsilon}\right)^{2n(n+1)},
\end{equation}
where $C_5>0$ is some constant.
\end{lemma}
\begin{proof}
See Appendix \ref{App: covering number model set lemma}.
\end{proof}

\begin{definition}{(Upper box counting dimension)}
\label{def: UBCD}
The upper box counting dimension of a set $S$ is defined as
\begin{equation}
    \text{dim}_\text{B}(S):=\limsup_{\epsilon\rightarrow 0}\log[\text{CN}\left(S,\lVert\cdot\rVert,\epsilon\right)]/\log[1/\epsilon].
\end{equation}
\end{definition}

\subsection{Proof of Theorem \ref{Thm: RIP Result}}\label{subsec: proof of main theorem}
\begin{proof}
To prove a RIP exists for the ICA model set $\modelset$ using the sketching operator $\mathcal{A}$ defined in (\ref{Eqn: Subgauss sketch operator}), we follow a similar line of argument to \cite{candes2011tight,rauhut2017low} by using an $\epsilon$-covering of $\mathfrak{N}\left(\modelset-\modelset\right)$ to extend the concentration results of the random Gaussian matrix $\mathbf{A}$ uniformally over the whole low-dimensional set. Specifically, we use the \textit{Recipe} framework proposed by Puy \textit{et al.} \cite{puyrecipes}, to formulate the compressive ICA RIP proof. The proof is separated by showing that the following assumptions hold:
\begin{itemize}
    \item[\textbf{(A1)}] The normalised secant set, denoted $\mathfrak{N}\left(\modelset-\modelset\right)$, has finite upper-box counting dimension $\text{dim}_\text{B}\left(\mathfrak{N}\left(\modelset-\modelset\right)\right)$ which is strictly bounded by $s\geq 1$, $\text{dim}_\text{B}\left(\mathfrak{N}\left(\modelset-\modelset\right)\right)<s$
    \item[\textbf{(A2)}] The sketching operator $\mathcal{A}$ satisfies concentration inequalities \cite{puyrecipes}. 
\end{itemize}
We begin with Assumption \textbf{(A1)}. Using Lemma \ref{lem: covering number of the normalized model set} and the definition of the upper box counting dimension in Definition \ref{def: UBCD}, it can be seen that $\text{dim}_\text{B}\left(\mathfrak{N}\left(\modelset-\modelset\right)\right)\leq 2n(n+1)$, so for any $s>2n(n+1)$ we satisfy Assumption \textbf{(A1)}. To prove Assumption \textbf{(A2)}, we have the following definition.
\begin{definition}{(Subguassian random variable)}
\label{def: subguass norm}
A subgaussian random variable $X$ is a random variable that satisfies
\begin{equation*}
    (\E\lvert X \rvert^q)^{1/q}\leq C_1\sqrt{q}\text{ for all } q\geq 1,
\end{equation*}
with $C_1>0$. The subgaussian norm of $X$, denoted by $\lVert X\rVert_{\Psi_2}$ is the smallest $C_1$ for which the last property holds, i.e.,
\begin{equation*}
    \lVert X\rVert_{\Psi_2}:=\text{sup}_{q\geq 1}\Big\{q^{-1/2}(\E\lvert X\rvert^q)^{1/q})\Big\}.
\end{equation*}
\end{definition}
\noindent Let $\mathbf{A}_i$ denote the $i$th row of the random Gaussian matrix $\mathbf{A}$. Then we use the fact \cite{vershynin2010introduction,puyrecipes} that 
\begin{equation}
    \lVert\mathbf{A}_i^T\vc\left(\mathcal{Z}\right)\rVert_{\Psi_2}\leq D\lVert\mathcal{Z}\rVert_F
\end{equation}
for all $\mathcal{Z}\in\mathfrak{C}$, where $D>0$ is an absolute constant.
Therefore $\Psi_2\leq D$ and Assumption $\textbf{A2}$ is satisfied. Finally, using Theorem 8 of \cite{puyrecipes}, we get the desired RIP result in Theorem \ref{Thm: RIP Result}.
\end{proof}

\subsection{Finite Sample Effects}
In practice, the sketch is constructed from a finite set of data $\left\{\mathbf{z}_i\right\}_{i=1}^{N}$ such that
\begin{equation}
    \label{Eqn: empirical CICA Sketch}
    \hat{\y}^{\textbf{w}}=\frac{1}{N}\sum^N_{i=1}\Phi^\textbf{w}\left(\z_i\right),
\end{equation}
where $\Phi^\textbf{w}(\cdot)$ is the feature function discussed in Section \ref{subsec: Compressive learning} acting on the whitened data $\z$. For compressive ICA we can explicitly define the feature function, acting on the whitened data, as
\begin{equation}
    \label{Eqn: CICA Feature Function}
    \Phi^\textbf{w}(\z)=\langle \mathbf{A}_j,\z^{\otimes^4}\rangle_F,
\end{equation}
for $j=1,\dots,m$, where $\mathbf{A}_j\in\R^{p}$ are the rows of a Gaussian matrix $\mathbf{A}$ and $\langle \cdot \rangle$ denotes the Frobenius inner product. Furthermore, for shorthand we denote $\z^{\otimes^4}=\z\otimes\z\otimes\z\otimes\z$. In other words, the feature function is taking random quartics of the data point $\z$. Note that the empirical sketch $\hat{\y}^\textbf{w}$ is equivalent to $\hat{\y}^\textbf{w}=\mathcal{A}\left(\hat{\mathcal{Z}}\right)$, as specified in (\ref{eqn: empirical sketch}), where $\hat{\mathcal{Z}}$ is the finite data approximation of the 4th order cumulant tensor $\mathcal{Z}$ defined by
\begin{equation}
\label{eqn: finite comp of tensor}
\begin{aligned}
    \hat{\mathcal{Z}}^4_{ijkl}=\frac{1}{N}\sum^N_{i,j,k,l=1}z_iz_jz_kz_l-\frac{1}{N^2}\sum^N_{i,j=1}z_iz_j\sum^N_{k,l=1}z_kz_l&-\frac{1}{N^2}\sum^N_{i,k=1}z_iz_k\sum^N_{j,l=1}z_jz_l\\
    &\,\,\,\,-\frac{1}{N^2}\sum^N_{i,l=1}z_iz_l\sum^N_{j,k=1}z_jz_k.
\end{aligned}
\end{equation} In this case, the error $\mathbf{e}$ defined in Theorem \ref{Thm: RIP Result} can be attributed to the finite sample effects of approximating the true 4th order cumulant tensor $\mathcal{Z}$ from finite data. We now state our final result of this section. 
\begin{theorem}[Finite Sample Effects]
\label{Thm: finite sample effects}
Assume the independent components $\s$ have finite non-zero kurtosis and have bounded support such that $\lVert\mathcal{S}\rVert_F\leq R$. Consider the ICA model $\z=\mathbf{Q}\s$ with corresponding 4th order cumulant tensor $\mathcal{Z}$ as in (\ref{eqn: first multilinear prop}). By considering any draw of $\z_1,\z_2,\dots,\z_N$ and associated 4th order cumulant tensor approximation $\hat{\mathcal{Z}}$, we have
\begin{equation}
        \lVert\mathbf{A}\vc\left(\mathcal{Z}\right)- \mathbf{A}\vc(\hat{\mathcal{Z}}) \rVert_2\leq \frac{R\sqrt{2(1+\delta)\log(1/\xi)}}{\sqrt{N}}
\end{equation}
with probability at least $1-\xi$ on the drawing of both $\z_i$'s and the random matrix $\mathbf{A}$.
\end{theorem}
\begin{proof}
See Appendix \ref{App: Proof of finite sample effects}.
\end{proof}


\subsection{Discussion}\label{subsec: unwhitened discussion}
The results in this section are all based on proving a RIP on the model set $\modelset$ defined in (\ref{eqn: model set}), where it is assumed the data $\x$ has been prewhitened to reduce the ICA model to $\z=\mathbf{Q}\s$ as discussed in Section \ref{subsec: ICA}. The prewhitening stage removes some of the degrees of freedom within the ICA inference task as it is necessary to estimate an orthogonal mixing matrix $\Q$. In some sketching cases, we may only see the data once, for example in the streaming context \cite{tropp2019streamingScientific}, and therefore prewhitening may not be possible. The fact that we are now estimating an arbitrary mixing matrix $\mathbf{M}$ instead of an orthogonal mixing matrix $\Q$ increases the degrees of freedom from $\frac{n(n+1)}{2}$ to $n(n+1)$. As a result, we must sketch the unwhitened moment tensor $\mathcal{X}$ such that
\begin{equation}
    \label{eqn: unwhitened sketch}
    \y^\mathbf{u}=\mathcal{A}\left(\mathcal{X}\right),
\end{equation}
where $\mathcal{A}(\cdot)=\mathbf{A}\vc(\cdot)$ and $\mathbf{A}\in\R^{m\times p}$ is a random matrix as defined in (\ref{Eqn: Subgauss sketch operator}). Here $\textbf{u}$ denotes that the sketch is acting on the unwhitened data $\x$. In addition, the feature function $\Phi^\textbf{u}(\cdot)$ for the unwhitened data can be defined as 
  \begin{align}
  \label{eqn: unwhitened feature function}
    \Phi^\textbf{u}(\x) &= \begin{bmatrix}
           \langle \mathbf{A}_j,\x^{\otimes^4}\rangle_F\\
          \x^{\otimes^2}
         \end{bmatrix},
  \end{align}
for $j=1,\dots,m$, where $\mathbf{A}_j\in\R^{p}$ are the rows of the matrix $\mathbf{A}$. Note that the feature function for the unwhitened data now includes quadratic moments\footnote{One could further reduce the size of the unwhitened sketch by instead computing random quadratic moments, however the reduction in complexity is minimal and therefore we leave this for future work.}, as well as random quartic moments, that are needed to estimate the mixing matrix $\mathbf{M}$ which has extra degrees of freedom. Recall from (\ref{eqn: prewhiten matrix}) that the mixing matrix $\mathbf{M}$ has the following decomposition \cite{de2000introduction2ICA}
\begin{equation}
    \mathbf{M}=\mathbf{P}\Q
\end{equation}
where $\mathbf{P}\coloneqq\mathbf{E}^T\Pi\mathbf{E}$ is the eigendecomposition of the covariance matrix $\mathbb{E}[\x\x^T]$ and $\mathbf{E}\in\R^{n\times n}$, $\Pi\in\R^{n\times n}$ are an orthogonal and diagonal matrix, respectively.

\section{CICA Algorithms}\label{sec: algorithm}
In this section we propose two distinct compressive ICA algorithms to estimate the mixing matrix $\mathbf{M}$ for both the whitened and unwhitened case.
 \subsection{Iterative Projection Gradient}\label{Subsec: IPG CICA}
 Iterative projection gradient (IPG) descent is a popular optimization scheme which enforces low dimensional structure e.g. sparsity, rank, etc, by projecting the object of interest onto the model set $\modelset$ after each subsequent gradient step. An iterative hard thresholding scheme was proposed in sparsity based compressive sensing \cite{blumensath2008iterative2,blumensath2009iterative},  where the smallest $n-k$ absolute entries are thresholded to zero to enforce the sparsity constraint and project the object onto the $k$-sparse model set. Blumensath \cite{BlumensathIPG11} shows that the thresholding operator is an orthogonal projection onto the $k$-sparse set thereby projecting to an element on the model set that is of minimal distance. For the case of compressive ICA, we also seek an orthogonal projection on to the ICA model set $\modelset$. Formally, we can define an orthogonal projection operator $\mathcal{P}_\modelset:\mathfrak{C}\mapsto\modelset$ of a 4th order cumulant tensor $\mathcal{Z}^*$ as 
 \begin{equation}
 \label{Eqn: Orthogonal Projection}
 \mathcal{P}_\modelset\left(\mathcal{Z}^*\right)\in\argmin_{\mathcal{Z}\in\modelset}\lVert\mathcal{Z}^*-\mathcal{Z}\rVert_F.
 \end{equation}
In other words, $\mathcal{P}_\modelset$ projects the object $\mathcal{Z}^*\in\mathfrak{C}$ onto the element in the model set that is of minimum distance w.r.t the Frobenius norm. In practice it is often difficult to find a projection operator that is both orthogonal and tractable in terms of computation. In \cite{cardoso1,cardoso2}, Cardoso showed that the ICA model set $\modelset\subseteq \mathfrak{R}\cap\mathfrak{L}$ where $\mathfrak{R}$ is the set of rank-$n$ tensors defined as 
\begin{equation}
    \label{eqn: set of rank-n tensors}
   \mathfrak{R}:=\{\mathcal{Z}\in\mathfrak{R}\mid \text{rank}(\Bar{\mathbf{Z}})=n\},
\end{equation}
where $\Bar{\mathbf{Z}}\in\R^{n^2\times n^2}$ is the matrix formed by rearranging the elements of the tensor $\mathcal{Z}$ into a $n^2\times n^2$ Hermitian matrix and where rank defines the standard matrix rank \cite{cardoso1}. 
Moreover, $\mathfrak{L}$ is the set of super-symmetric tensors defined by
\begin{equation}
    \label{eqn: set of supersym tensors}
    \mathfrak{L}:=\{\mathcal{Z}\in\mathfrak{L}\mid \mathcal{Z}_{q(ijkl)}=\mathcal{Z}_{ijkl}\}
\end{equation}
where $q$ defines all permutations of the index $ijkl$. In fact, Cardoso proved in \cite{cardoso2} that locally the converse is true, for instance within some neighbourhood of $\mathcal{Z}$ the following holds:
\begin{equation}
    \label{eqn: converse cardoso}
    \mathfrak{R}\cap\mathfrak{L}\subseteq\modelset.
\end{equation}
Therefore, within some neighbourhood of $\mathcal{Z}^*$, projecting onto the ICA model set $\modelset$ is equivalent to projecting onto $\mathfrak{R}\cap\mathfrak{L}$. Moreover, in \cite{Cadzow1988SignalEC}, Cadzow proved that alternate projections onto $\mathfrak{R}$ and $\mathfrak{L}$ is guaranteed to converge onto the intersection\footnote{In general, rank forcing destroys symmetry while symmetrization destroys the rank-$n$ property, therefore alternate projections are needed until convergence.} $\mathfrak{R}\cap\mathfrak{L}$. Fundamentally, the projections onto $\mathfrak{R}$ (rank-$n$ approximation) and $\mathfrak{L}$ (averaging over permutations), denoted by $\mathcal{P}_\mathfrak{R}$ and $\mathcal{P}_\mathfrak{L}$ respectively, are both simple to compute and are orthogonal. Alternate orthogonal projections onto $\mathfrak{R}$ and $\mathfrak{L}$ ensures a stable projection onto $\mathfrak{R}\cap\mathfrak{L}$ \cite{Cadzow1988SignalEC} which locally, results in an orthogonal projection onto the ICA model set $\modelset$. Formally, we define the orthogonal projection $\mathcal{P}_\modelset$ below in Algorithm \ref{Alg: Orthogonal Projection}. In practice, Algorithm \ref{Alg: Orthogonal Projection} converges to below a small tolerance in very few iterations ($\sim 10$ iterations). We can now state our full CICA IPG algorithm detailed in Algorithm \ref{Alg: CICA IPG}. Here the step size $\mu_j$ is computed optimally to guarantee convergence \cite{BlumensathIPG11,blumensath2009iterative}, $\mathcal{A}^*$ denotes the adjoint sketching operator and $\beta$ is a fixed shrinking step size parameter.

\begin{algorithm}
\caption{$\mathcal{P}_\modelset$ : Orthogonal Projection onto ICA Model Set}
\begin{algorithmic}
\Require Cumulant tensor $\mathcal{Z}^*\in\mathfrak{C}$
\While{Not Converged}
\State Project onto $\mathfrak{R}$: $\mathcal{Z}^1=\mathcal{P}_\mathfrak{R}(\mathcal{Z})$
 (Matricize $\mathcal{Z}$ into a $n^2\times n^2$ Hermitian matrix and take a rank-$n$ approximation using truncated SVD)
 \State Project onto $\mathfrak{L}$: $\mathcal{Z}^2=\mathcal{P}_\mathfrak{R}(\mathcal{Z}^1)$ (Average across all permutations of $q(ijkl)$ for all indices $ijkl$)
\EndWhile
\end{algorithmic}
\label{Alg: Orthogonal Projection}
\end{algorithm}

\begin{algorithm}
\caption{$\texttt{CICA}_\texttt{IPG}$ : Iterative Projection Gradient Descent Compressive ICA}
\begin{algorithmic}
\Require Initialisation $\mathcal{Z}^0$, tolerance $\epsilon$ and shrinking parameter $\beta$.
\While{ $\lVert \y^{\normalfont\textbf{w}}-\mathcal{A}\big(\mathcal{Z}^j\big) \rVert^2_2 > \epsilon$}
\State Compute $\mu_j=\dfrac{\lVert\mathcal{A}^*\big(\y^\textbf{w}-\mathcal{A}\big(\mathcal{Z}^j\big)\big)\rVert^2_F}{\lVert\y^\textbf{w}-\mathcal{A}\big(\mathcal{Z}^j\big)\rVert^2_2}$
\While{${\lVert\y^{\normalfont\textbf{w}}-\mathcal{A}\big(\mathcal{Z}^{j+1}\big)\rVert^2_2}>{\lVert\y^{\normalfont\textbf{w}}-\mathcal{A}\big(\mathcal{Z}^j\big)\rVert^2_2}$}
\State $\mu_j \gets \beta\mu_j$
\State $\mathcal{Z}^{j+\frac{1}{2}}\gets \mathcal{Z}^j+\mu_j\mathcal{A}^*\left(\y^\textbf{w}-\mathcal{A}\left(\mathcal{Z}^{j}\right)\right)$
\State $\mathcal{Z}^{j+1}\gets\mathcal{P}_\modelset\left(\mathcal{Z}^{j+\frac{1}{2}}\right)$
\EndWhile
\EndWhile
\end{algorithmic}
\label{Alg: CICA IPG}
\end{algorithm}

\subsubsection{Unwhitened IPG}
It was discussed in Section \ref{subsec: unwhitened discussion} that it is often convenient, from an online processing point of view, to directly sketch the unwhitened data $\x$. Using the properties of the matrix-tensor product \cite{de1997signal}, it can be seen that 
\begin{equation}
    \mathbf{A}\vc\left(\mathcal{X}\right) = \mathbf{A}\bar{\mathbf{P}}\vc\left(\mathcal{Z}\right),
\end{equation}
where $\bar{\mathbf{P}}\coloneqq \mathbf{P}\otimes\mathbf{P}\otimes\mathbf{P}\otimes\mathbf{P}$. As defined in (\ref{eqn: unwhitened feature function}), the unwhitened feature function $\Phi^\textbf{u}$ includes the second order moment of $\x$, namely $\x^{\otimes^2}$. The empirical sketch $\hat{\mathbf{y}}^\textbf{u}$ therefore includes the sample covariance $\hat{\Sigma}\coloneqq\frac{1}{N}\sum^N_{i=1}\x_i^{\otimes^2}$, which can be used to estimate an approximation of $\mathbf{P}$, denoted $\hat{\mathbf{P}}$, by using the eigenvalue decomposition of $\hat{\Sigma}$ \cite{comon1994independent} at the beginning of Algorithm \ref{Alg: CICA IPG}. By denoting $\hat{\bar{\mathbf{P}}}\coloneqq \hat{\mathbf{P}}\otimes\hat{\mathbf{P}}\otimes\hat{\mathbf{P}}\otimes\hat{\mathbf{P}}$, the gradient step in Algorithm \ref{Alg: CICA IPG} can be replaced by
\begin{equation}
      \mathcal{Z}^{j+\frac{1}{2}}=\mathcal{Z}^j+\mu_j\mathbf{A}^T\big({\y}^\textbf{u}-\mathbf{A}{\hat{\bar{\mathbf{P}}}}\big(\mathcal{Z}^{j}\big)\big),
\end{equation}
as well as the associated step size $\mu_j$ and stopping criteria. As a result, the CICA IPG algorithm proceeds as normal by employing the original orthogonal projection $\mathcal{P}_\modelset$.
\subsection{Alternating Steepest Descent}
The second proposed algorithm in the way of alternating steepest descent (ASD) is inherently different from the IPG scheme previously discussed. To see why, it is insightful to rewrite (\ref{eqn: initial constrained l_2}) in terms of the elements of the product set $\mathfrak{D}$ and O$(n)$:
\begin{equation}
    \label{eqn: ICA optimisation alternative}
    \min_{\substack{\Q^T\Q=I\\\mathcal{S}\in\mathfrak{D}}}F\left(\mathcal{S},\mathbf{Q}\right)=\lVert\y^\textbf{w}-\mathcal{A}(\mathcal{S}\times_1\Q\times_2\Q\times_3\Q\times_4\Q)\rVert^2_2,
\end{equation}
where we have used the multilinear property discussed in (\ref{eqn: first multilinear prop}). As the optimization problem is now explicitly defined by the mixing matrix $\Q$ and a sparse diagonal tensor $\mathcal{S}$, it is sufficient to optimise with respect to these parameters in an alternating steepest descent scheme. This approach contrasts the IPG scheme, as once we initialise the mixing matrix $\Q$ and the diagonal cumulant tensor $\mathcal{S}$ appropriately, then we can optimise directly on the model set $\modelset$. We can initially state the ASD steps:
\begin{enumerate}
    \item $\mathcal{S}^*=\min_{\mathcal{S}\in\mathfrak{D}}F(\mathcal{S},\Q)$
    \item $\Q^*=\min_{\Q^T\Q=I}F(\mathcal{S}^*,\Q)$
\end{enumerate}
\noindent Note that the diagonal cumulant tensor $\mathcal{S}\in\mathfrak{D}$ can be simply reformulated as an $n$ sparse vector with known support, therefore one can perform element-wise differentiation on the $n$ entries $\mathcal{S}_{iiii}$ for $i=1:n$. The second step requires more attention as we have the constraint $\Q^T\Q=\mathbf{I}$ (i.e. $\Q\in\text{O}(n)$). The set of $n\times n$ orthogonal matrices is an instance of a Stiefel manifold \cite{wotaoStiefel}, therefore $F$ is minimized directly on the Stiefel manifold.
\subsubsection{Stiefel Manifold Optimisation}
Given a feasible matrix $\Q$ and the gradient $\nabla_\Q F=\Big(\frac{\partial F(\mathcal{S},\Q)}{\partial\Q_{ij}}\Big)$, define a skew-symmetric matrix $\mathbf{B}$ as
\begin{equation}
    \mathbf{B}=\nabla_\Q F \Q^T-\Q(\nabla_\Q)^T.
\end{equation}
The update on the Stiefel manifold is determined by the Crank-Nicholson scheme \cite{crank1947practical} denoted
\begin{equation}
    Y(\tau)=\Q-\frac{1}{2}\mathbf{B}(\Q+Y(\tau))
\end{equation}
where $Y(\tau)=(I-\frac{\tau}{2}\mathbf{B})^{-1}(I-\frac{\tau}{2}\mathbf{B})\Q$. The matrix $(I-\frac{\tau}{2}\mathbf{B})^{-1}(I-\frac{\tau}{2}\mathbf{B})$ is referred to as the Cayley transform \cite{wotaoStiefel} of $\mathbf{B}$. The descent curve $Y(\tau)$ has the following useful features
\begin{itemize}
    \item $Y(\tau)$ is smooth on $\tau$
    \item $Y(0)=\Q$
    \item $Y(\tau)^TY(\tau)=\Q^T\Q$ for all $\tau\in\R$.
\end{itemize}
As a result, we perform a steepest descent on $\Q$ with line search along the descent curve $Y(\tau)$ with respect to $\tau$. For more details on optimisation methods constrained to the Stiefel manifold refer to \cite{wotaoStiefel}. We can now state our second proposed CICA algorithm in Algorithm \ref{Alg: CICA GD}.

\begin{algorithm}
\caption{$\texttt{CICA}_\texttt{ASD}$ : Alternating Steepest Descent Compressive ICA }
\begin{algorithmic}
\Require Initialisation $\mathcal{Z}^0=\mathcal{S}^0\times_1\Q^0\times_2\Q^0\times_3\Q^0\times_4\Q^0$, tolerance $\epsilon$ and step size $\mu$.
\While{ $\lVert \y^{\normalfont\textbf{w}}-\mathcal{A}\big(\mathcal{Z}_j\big) \rVert^2_2 > \epsilon$}
\State $\mathcal{S}^{j+1}=\mathcal{S}^j+\mu\nabla_\mathcal{S} F\left(\mathcal{S}^j,\Q^j\right)$
\While{Perform line search}
\State $Y(\tau)=\Q-\frac{\tau}{2}\mathbf{B}\left(\Q+Y\left(\tau\right)\right)$
\State $\Q^{t+1}\gets Y\left(\tau^*\right)$
\EndWhile
\State $\mathcal{Z}^{j+1}\gets\mathcal{S}^{j+1}\times_1\Q^{j+1}\times_2\Q^{j+1}\times_3\Q^{j+1}\times_4\Q^{j+1}$
\EndWhile
\end{algorithmic}
\label{Alg: CICA GD}
\end{algorithm}

\subsubsection{Practicalities}\label{subsubsec: practicalities}
We start by stating the computational complexity of each proposed CICA algorithm. Here we assume that a fast SRHT, as discussed in \ref{subsubsec: SRHT sec}, is used to compute the sketch. For the IPG scheme, the symmetry projection $\mathcal{P}_\mathfrak{L}$ costs $\mathcal{O}(n^4)$ flops through averaging along all index permutations. A rank-$r$ approximation of a general matrix $\mathbf{X}\in\R^{m\times n}$ costs $\mathcal{O}(r^2(n+m))$ flops \cite{WOOLFE2008335}, therefore the rank projection operator $\mathcal{P}_\mathfrak{R}$ costs a total of $\mathcal{O}(n^4)$ flops. The gradient step in Algorithm \ref{Alg: CICA IPG} costs a total of $\mathcal{O}(p\log(m))$ flops due to the use of the sketching operator $\mathcal{A}(\mathcal{Z}^j)$ at each iteration which results in the IPG algorithm therefore having a total cost of $\mathcal{O}(p\log(m)+n^4))$ flops. In the second proposed ASD algorithm, the gradient step in terms of the diagonal tensor in Algorithm \ref{Alg: CICA GD}, again has a cost of $\mathcal{O}(p\log(m))$ flops. The line search $Y(\tau)$ costs a total of $\mathcal{O}(n^3)$ flops \cite{wotaoStiefel} resulting in the ASD algorithm having a computational complexity of $\mathcal{O}(p\log(m)+n^3)$. Note that both proposed CICA algorithms have computational complexity that is independent of the length of the data $N$ which can be extremely large for modern day applications. 
\par As is the case for the general ICA problem, the compressive ICA optimisation problem is non-convex and both algorithms proposed may be prone to converging to local minima. As a result, we consider the option of possible restarts at random initialisations to obtain a good solution. We also consider a proxy projection operator that uses a Given's rotation scheme, popular in many ICA algorithms (see \cite{comon1994independent,cardoso1993blind}), that approximately diagonalises the cumulant tensor $\mathcal{Z}$ with respect to some contrast function, followed by thresholding the cross cumulants of that approximately diagonalised tensor to zero \cite{sheehan2019compressive}. We have observed in practice that this proxy projection operator is less sensitive to the non-convex landscape of the optimization problem, which could be explained by the robustness of Given's rotations \cite{comon1994independent}, hence multiple restarts are rarely required. The proxy projection operator, which we denote by $\hat{\mathcal{P}}_\modelset$, costs $\mathcal{O}(n^4)$ flops for the Given's rotation scheme to approximately diagonalise the cumulant tensor \cite{comon1994independent}, and $\mathcal{O}(n^4-n)$ flops for the thresholding of the cross-cumulants. Therefore in total the proxy IPG algorithm has approximately the same computational complexity as our previous IPG algorithm.

\section{Empirical Results}\label{sec: results}
\subsection{Phase Transition}\label{subsec: Phase tran}
Phase transitions are an integral part of analysis that are used frequently in the compressive sensing literature \cite{phase_tran_Lotz} to show a sharp change in the probability of successful reconstruction of the low dimensional object as the sketch size $m$ increases. The location at which the phase transition occurs can provide a tight bound on the required sketch size needed given the number of independent components $n$ and further consolidates the theoretical bound of the RIP derived in Section \ref{sec: CICA Theory}. To set up the phase transition experiment, we constructed the expected cumulant tensor $\mathcal{S}$ of $n$ Laplacian sources and transformed the tensor with an orthogonal mixing matrix $\mathbf{M}$ using the multilinear property in (\ref{eqn: first multilinear prop}), resulting in an expected cumulant tensor $\mathcal{Z}$. For each number of independent components $n$, 250 Monte Carlo simulations on the mixing matrix $\mathbf{M}$ were executed for increasing sketch size $m$ between 2 and 700. A successful reconstruction was determined if the Amari error\footnote{The Amari error is used widely in the ICA literature as it is both scale and permutation invariant, which are the two inherent ambiguities of ICA inference.} \cite{amari1996new} between the true mixing matrix $\mathbf{M}$ and the estimated mixing matrix $\hat{\mathbf{M}}$, defined by
\begin{equation}
\begin{aligned}
    \label{Eqn: Amari error}
    d(\mathbf{M},\hat{\mathbf{M}})=&\frac{1}{2n}\sum^n_{i=1}\Bigg(\frac{\sum^n_{j=1}\lvert b_{ij}\rvert}{\text{max}_j\lvert b_{ij}\rvert}-1 \Bigg)+\frac{1}{2n}\sum^n_{j=1}\Bigg(\frac{\sum^n_{i=1}\lvert b_{ij}\rvert}{\text{max}_i\lvert b_{ij}\rvert}-1 \Bigg),
\end{aligned}
\end{equation}
was smaller than $ d(\mathbf{M},\hat{\mathbf{M}})\leq 10^{-6}$, where $b_{ij}=(\mathbf{M}\hat{\mathbf{M}}^{-1})_{ij}$. The probability of successful reconstruction was given by the number of successful reconstructions within the 250 Monte-Carlo tests. We use the IPG version of the CICA algorithm for these results, although the ASD version provides nearly exactly the same results. It is insightful to begin by fixing the number of sources, here $n=8$, to highlight the sharp transition as shown in Figure \ref{Fig: PhaseTran Solo}. We highlight some important bounds including the multiples of 2 and 4 times the dimension of the model set $\modelset$, depicted by the orange lines. For comparison, the dimension of the space of cumulant tensors $\mathfrak{C}$, in other words the size of the cumulant tensor, is shown by the red line. The phase transition occurs in between 2 and 4 times the model set dimension indicating that choosing $m\geq2n(n+1)$ would be sufficient in successfully inferring the mixing matrix with high probability. 
\par Figure \ref{Fig: PhaseTran All} generalises the single phase transition result for the number of independent components varying between $n=2$ and $n=10$. Once again, the important bounds of the model set dimension (green), 2 and 4 multiples of the model set dimension (orange) and the dimension of the space of cumulant tensors (red) are shown. Figure \ref{Fig: PhaseTran All} explicitly shows that the phase transition empirically occurs within the location of $m=n(n+1)$ and $m=2n(n+1)$ and provides us with a tight practical lower bound of  $m\geq2n(n+1)$ on the sketch size for successful inference of the mixing matrix with high probability. Recall that in Theorem \ref{Thm: RIP Result}, the RIP holds when $m\geq  2n(n+1)$. The location of the phase transition in the empirical results therefore further consolidates the theoretical result. For a given number of independent components $n$, the ratio between the upper orange line (4 times the model set dimension) and the red line (space of cumulant tensor dimension) provides a realistic compression rate in comparison to using the whole cumulant tensor of which many ICA techniques use. Importantly, as the number of independent components increases the ratio between these two lines decreases, resulting in further compression.

\begin{figure}[ht]
\centering
 \includegraphics[width=4.1in]{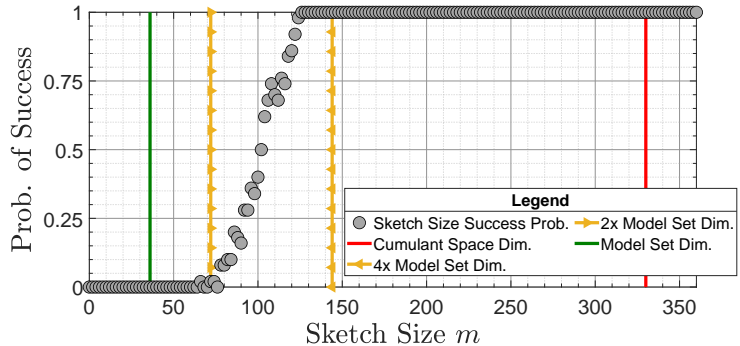}
\caption{A phase transition between unsuccessful and successful mixing matrix inference as the sketch size $m$ increases and the number of independent components is fixed at $n=8$.}
\label{Fig: PhaseTran Solo}
\end{figure}

\begin{figure}[ht]
\centering
 \includegraphics[width=4.8in]{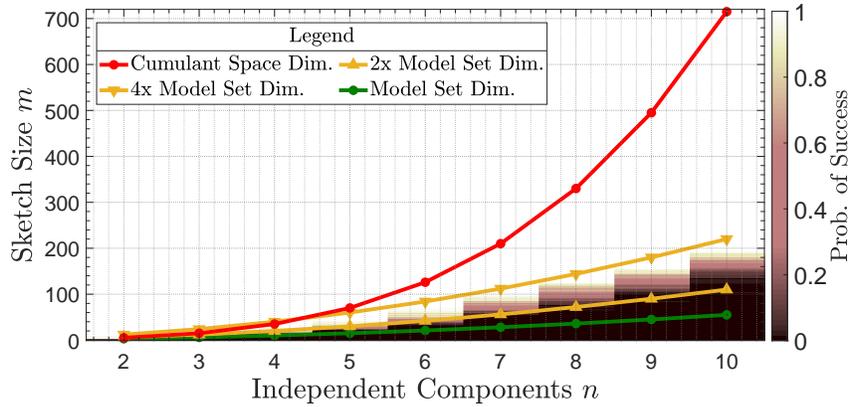}
\caption{A phase transition between unsuccessful and successful mixing matrix inference as the sketch size $m$ and the number of independent components $n$ increases.}
\label{Fig: PhaseTran All}
\end{figure}

\subsection{Statistical Efficiency}\label{subsec: statistical eff}
As was shown in Section \ref{subsec: Phase tran}, the potential compression rates of sketching the cumulant tensor are high which can lead to a significantly reduced memory requirement. In this section we numerically analyse the trade-off between the sketch size and the loss of information. Statistical efficiency is a measure of the variability or quality of an unbiased estimator \cite{FisherEfficiency}, where the Cram\'er-Rao bound provides a lower bound on the variability of an estimator and gives a best case scenario. For fair comparison, we instead use the variability of an estimator inferred by an algorithm that explicitly makes use of the cumulant information, just as the proposed CICA algorithms, to infer the mixing matrix estimate. As such, we use Comon's ICA algorithm, detailed in \cite{comon1994independent}, that minimizes a kurtosis based contrast function using a sequence of Given's rotations on pairwise cumulants as the approximate full data bound (e.g. no compression). We could have equivalently used the well-known Joint Approximation Diagonalization of Eigen-matrices (JADE) algorithm \cite{cardoso1993blind} or any other cumulant based ICA algorithm as the approximate bound, which gives similar results. To this end, we make use of the relative efficiency, defined as
\begin{equation}
    \label{eqn: rel eff measure}
    e(\M_1,\M_2)=\frac{\text{var}\left(d(\M_\theta,\M_1)\right)}{\text{var}\left(d(\M_\theta,\M_2)\right)},
\end{equation}
where $d(\cdot,\cdot)$ is the Amari Error defined in (\ref{Eqn: Amari error}) and $\M_\theta$ is the true mixing matrix. Denoting $\M_{\text{FD}}$ and $\M_{\text{CICA}}$ as the mixing matrix estimates of Comon's ICA algorithm (full data) and the proposed CICA algorithm, respectively, we expect $0\leq e\left(\M_{\text{FD}},\M_{\text{CICA}}\right)\leq 1$ as the Comon algorithm exhibits no compression and makes use of the full cumulant tensor available. As the relative efficiency $e\left(\M_{\text{FD}},\M_{\text{CICA}}\right)$ approaches 1, the sketch estimate becomes more statistically efficient. We perform our efficiency test on $n=6$ independent components of signal length $N=1000$. The signal length does not affect the results as the dependence of $N$ drops out of the relative efficiency measure, for example see \cite{sheehan2019compressive}. For each of the 100 Monte-Carlo simulations, the $n=6$ independent components are randomly sampled \cite{bach2002kernel} from a range of distributions with unique characteristics that are shown in Figure \ref{Fig: efficiency distributions}. The true mixing matrix $\M_\theta$ was sampled once and fixed throughout. For each sketch size $m$, 100 simulations were executed where the mixing matrix was estimated and the Amari error was calculated. The variance of the Amari errors was compared with the full data counterpart and plotted as the relative efficiency in Figure \ref{Fig: Relative Efficiency plot}. Figure \ref{Fig: Relative Efficiency plot} shows the relative efficiency as the sketch size $m$ increases. As $m$ increases the relative efficiency approaches 1 (i.e. as statistically efficient as using the full cumulant tensor with no compression). It is evident that there is a trade-off between the rate of compression and the statistical efficiency, for instance, the smaller the sketch size the greater the loss of statistical efficiency. This is to be expected as the harsher you compress the data the more loss of information you experience. Nontheless, the tradeoff is controlled, for example, a sketch of size $m=100$ has a drop of around $40\%$ of efficiency. 

\begin{figure}[ht]
\centering
\includegraphics[width=5in]{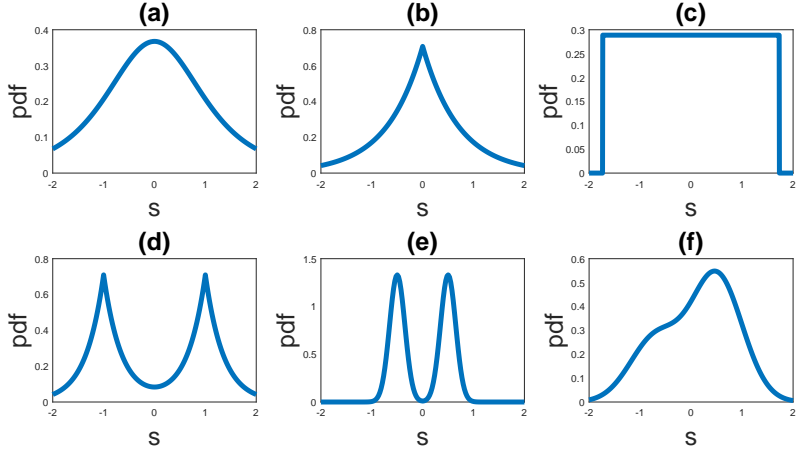}
\caption{(a) Student's $t$ distribution $(\nu=3)$ (b) Laplace distribution $(\mu=0,b=1)$ (c) continuous uniform distribution $(a=-\sqrt{3},b=\sqrt{3})$ (d) mixture of 2 Laplaces $(\mu_1,\mu_2=-1,1\,\,b_1=b_2=1)$ (e) symmetric bimodal mixture of Gaussians $(\mu_1,\mu_2=-1,1\,\, \sigma_1=\sigma_2=0.15)$ (f) asymmetric unimodal mixture of Gaussians $(\mu_1,\mu_2=-0.7,0.5\,\, \sigma_1=\sigma_2=0.5)$}
\label{Fig: efficiency distributions}
\end{figure}

\begin{figure}[ht]
\centering
\includegraphics[width=4.5in]{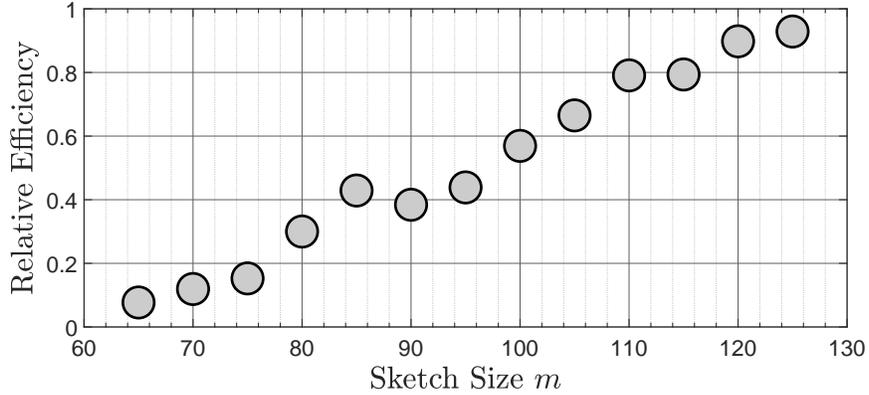}
\caption{The relative efficiency of the full data cumulant tensor (Comon's ICA) and sketch mixing matrix estimates for increasing sketch size $m$.}
\label{Fig: Relative Efficiency plot}
\end{figure}

\subsection{Cylinder Velocity Field }\label{subsec: Cylinder experiment}
We analyse and compare the proposed CICA scheme on a dataset consisting of a flow field around a cylinder obstruction as depicted in Figure \ref{Fig: cylinder velocity}. Using ICA, one can obtain a model that describes the fluctuations of the streamwise velocity field around it's mean value as a function of time. Details of the experimental set up can be seen in \cite{cylinderdataset_original,cylinderdataset_original2}. The dataset is of size $\mathbf{X}\in\mathbb{R}^{14400\times 100}$ consisting of 14400 spatial locations over 100 time intervals. Here we compare our proposed CICA scheme with the well-known fast ICA algorithm \cite{FastICA_ref}, as well the JADE \cite{cardoso1993blind} and Comon algorithm \cite{comon1994independent} which, like the proposed CICA scheme, are cumulant based.An initial prewhitening stage inferred the prewhiten matrix $\mathbf{P}\in\mathbb{R}^{100\times 8}$. Each algorithm then estimated the $\mathbf{Q}\in\mathbb{R}^{8\times 8}$, resulting in a mixing matrix estimate $\mathbf{M}=\mathbf{P}\mathbf{Q}$. For the proposed CICA scheme, the IPG version was used with a SRHT matrix $\mathbf{A}$, however ASD version produces similar reconstructions. Figure \ref{Fig: stream velocity compare} shows the 8 independent components which describe the fluctuations of the streamwise velocity around the cylinder obtained by Fast ICA, JADE, Comon and CICA, respectively. For our proposed CICA algorithm, a sketch of size $m=114$ is used. Visually comparing the reconstructions, one can see that the CICA algorithm performs competitively with negligible artifacts present. Moreover, the CICA scheme achieves a compression rate of approximately 3 in comparison to the other cumulant based ICA methods discussed. \par Next, we compare the effect of the sketch size on the resulting reconstructions. A sketch size of $m=72,108$ and $144$ are considered with the reconstructions shown in Figure \ref{Fig: stream velocity sketch compare}. For $m=108$, the sketch is of sufficient size to successfully identify the unique fluctuations of the velocity field, however, due to the harsher compression rate some notable artefacts are present. For example, in the first and third fluctuations there are some oscillating type artifacts which can be attributed to the higher frequencies in the system. Furthermore, the sketch of size $m=72$ fails to identity the main fluctuations of the velocity field.

\begin{figure}[t!]
\centering
\includegraphics[width=4.8in]{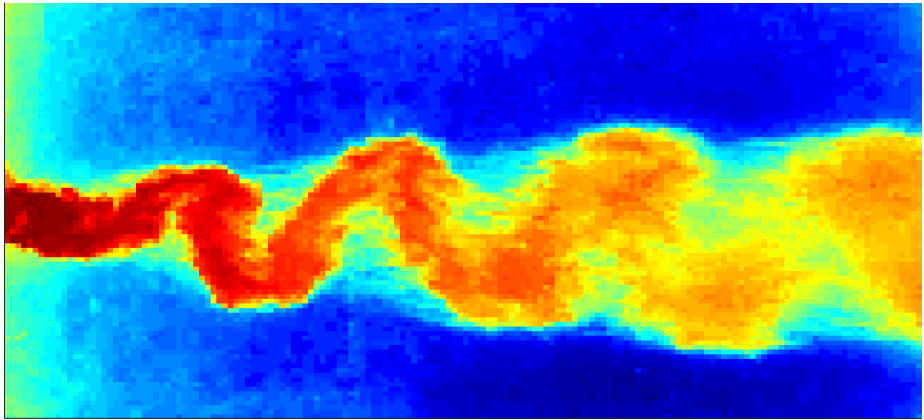}
\caption{The figure shows the velocity field around a cylinder for a fixed point in time.}
\label{Fig: cylinder velocity}
\end{figure}

\begin{figure}[t!]
\centering
\includegraphics[width=5.5in]{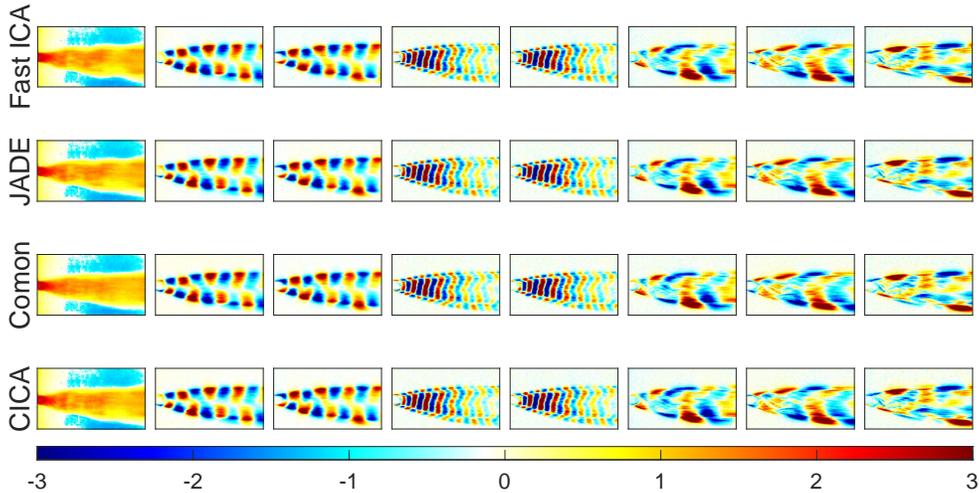}
\caption{From left to right the dominant fluctuations of the streamwise velocity field. From top to bottom the Fast ICA, JADE, Comon and CICA reconstructions.}
\label{Fig: stream velocity compare}
\end{figure}

\begin{figure}[t!]
\centering
\includegraphics[width=5.5in]{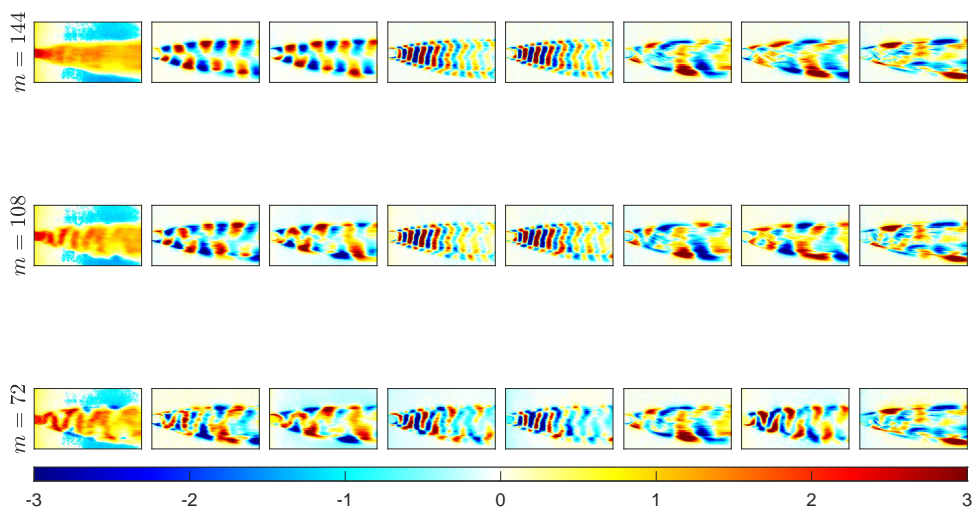}
\caption{The figure shows the effect of the sketch size on the reconstruction of the fluctuations. From top to bottom a sketch size of $m=144,108$ and $72$. }
\label{Fig: stream velocity sketch compare}
\end{figure}

\section{Conclusion}
\par In this paper we initially showed that a low dimensional model set exists for the ICA problem. It was demonstrated theoretically that a RIP exists for the ICA model using Gaussian ensembles provided the sketch size was set proportionally to the model set dimensions, which in turn induced the existence of an instance optimal decoder. The theoretical results were empirically validated by showing the location of a sharp phase transition between a state of unsuccessful inference to a state of successful inference of the ICA mixing matrix as the sketch size increased. Using both synthetic and real data, we analysed the robustness of the proposed CICA algorithms and highlighted the effect of choosing the sketch size $m$. Furthermore, the particular branch of compressive learning was discussed that consists of sketching distribution free models (e.g. PCA, ICA) that leverage some intermediary statistic space, here the space of cumulant tensors, to form the sketch. This poses some interesting open questions on how to design a sketch given other distribution free models and how the low dimension nature of the model set manifests itself structurally, in terms of sparsity, low rank, etc. to construct a practical sketching decoder.  

\section*{Acknowledgements}
This work was supported by the ERC Advanced grant, project C-SENSE (ERC-ADG-2015-694888). Mike E. Davies is also supported by a Royal Society Wolfson Research Merit Award.
\section*{Data Availability Statement}
The data used in Section \ref{subsec: Cylinder experiment} is available at the repository
\url{https://github.com/jonnyhigham/POD_DMD}.

\section*{Code Availability}
A MATLAB implementation of the proposed CICA algorithms are available at the repository  \url{https://gitlab.com/mpsheehan1995/CICA}.

\footnotesize
\bibliographystyle{unsrt}
\bibliography{main.bib}

\newpage
\appendix

\section{Proof of Lemma \ref{lem: covering number of the normalized model set} }\label{App: covering num Appendix A}
To prove Lemma \ref{lem: covering number of the normalized model set}, we use a similar line of argument to Clarkson in \cite{clarkson2008tighter} by splitting the normalized secant set into the set of short and long secants parametrized by a distance $\eta$. First we state an important lemma on covering the model set intersected with the unit sphere in $\mathbb{R}^{\bar{n}}$, where $\bar{n}=n^4$, denoted by $\Bar{\Bar{\mathfrak{S}}}_\mathcal{H}\coloneqq\mathfrak{S}_\mathcal{H}\cap \mathbb{S}^{\bar{n}-1}$ (e.g. $\left\Vert\mathcal{Z}\right\Vert_F=1$), that will used later in the proof.

\begin{lemma}[Covering number of $\Bar{\Bar{\mathfrak{S}}}_\mathcal{H}$]
\label{App: covering number model set lemma}
The covering number of $\Bar{\Bar{\mathfrak{S}}}_\mathcal{H}$ with respect to the Frobenius norm $\lVert \cdot \rVert_F$ is 
\begin{equation}
    \text{CN}\left(\mathfrak{S'},\lVert\cdot\rVert_F,\epsilon\right)\leq \left(\frac{6}{\epsilon}\right)^{n(n+1)}
\end{equation}
\end{lemma}

\begin{proof}
Recall that $\mathcal{Z}\in\Bar{\Bar{\mathfrak{S}}}_\mathcal{H}$ has the decomposition $\mathcal{Z}=\mathcal{S}\times_1\mathbf{Q}\times_2\mathbf{Q}\times_3\mathbf{Q}\times_4\mathbf{Q}$ such that $\lVert\mathcal{Z}\rVert_F=1$ where $\mathcal{S}\in\mathfrak{D}$ and $\Q\in\text{O}(n)$. As the Frobenius norm is rotationally invariant then the following holds $\lVert\mathcal{Z}\rVert_F=\lVert\mathcal{S}\rVert_F=1$ for all $\mathcal{Z}\in\Bar{\Bar{\mathfrak{S}}}_\mathcal{H}$. Our argument constructs an $\epsilon$-net  for $\Bar{\Bar{\mathfrak{S}}}_\mathcal{H}$ by covering the sets $\mathfrak{D}$ and $\text{O}(n)$ respectively. As $\lVert\mathcal{Z}\rVert_F=1 \implies \lVert\mathcal{S}\rVert_F=1$, it is sufficient to consider $\bar{\bar{\mathfrak{D}}}\coloneqq\mathfrak{D}\cap \mathbb{S}^{n-1}$. Then we take $\underline{\smash{\bar{\bar{\mathfrak{D}}}}}$ to be an $\epsilon/2$- net for $\bar{\bar{\mathfrak{D}}}$. As $\bar{\bar{\mathfrak{D}}}$ is a $n$ dimensional subspace, then $$\text{CN}\left(\bar{\bar{\mathfrak{D}}},\lVert\cdot\rVert_F,\epsilon/2\right)\leq \Big(\frac{6}{\epsilon}\Big)^n.$$ Next, we cover the set of $n\times n$ orthogonal matrices denoted $\text{O}(n)$. We follow a similar argument to \cite{rauhut2017low,candes2011tight} by letting $\text{Q}(n):=\{\mathbf{X}\in\R^{n\times n} : \lVert\mathbf{X}\rVert_{1,2}\leq 1\}$, where 
\begin{equation*}
    \lVert\mathbf{X}\rVert_{1,2} = \max_i \lVert X(:,i)\rVert_{2}
\end{equation*}
is the maximum column norm of a matrix $\mathbf{X}$. It is straightforward to see that $\text{O}(n)\subset\text{Q}(n)$ since the columns of an orthogonal matrix are unit normed. It can be seen in \cite{candes2011tight} that an $\epsilon/2$-net $\text{O}(n)$, denoted by $\underline{\smash{\text{O}}}(n)$, has a covering number 
$$\text{CN}(\text{O}(n),\lVert\cdot\rVert_{1,2},\epsilon/2)\leq \Big(\frac{6}{\epsilon}\Big)^{n^2}.$$
Now let $\underline{\smash{\Bar{\Bar{\mathfrak{S}}}}}_{\mathcal{H}}:=\{\underline{\smash{\mathcal{S}}}\times_1\underline{\smash{\Q}}\times_2\underline{\smash{\Q}}\times_3\underline{\smash{\Q}}\times_4\underline{\smash{\Q}}:\underline{\smash{\mathcal{S}}}\in \underline{\smash{\bar{\bar{\mathfrak{D}}}}},\underline{\smash{\Q}}\in\underline{\smash{\text{O}}}(n)\}$, and remark that\\
\begin{align*}
    \text{CN}\big(\bar{\bar{\mathfrak{S}}}_\mathcal{H},\lVert\cdot\rVert_F,\epsilon\big)&\leq \text{CN}\big(\bar{\bar{\mathfrak{D}}},\lVert\cdot\rVert_F,\epsilon/2\big)\;\; \text{CN}\big(\text{O}(n),\lVert\cdot\rVert_{1,2},\epsilon/2\big)\\
    &\leq \Big(\frac{6}{\epsilon}\Big)^{n(n+1)}.
\end{align*}
It remains to show that for all $\mathcal{Z}\in\bar{\bar{\mathfrak{S}}}_\mathcal{H}$ there exists $\underline{\smash{\mathcal{Z}}}\in\underline{\smash{\bar{\bar{\mathfrak{S}}}}}_\mathcal{H}$ such that $\lVert\mathcal{Z}-\underline{\smash{\mathcal{Z}}}\rVert_F\leq\epsilon$.

\par Fix $\mathcal{Z}\in\bar{\bar{\mathfrak{S}}}_\mathcal{H}$ and note the decomposition $\mathcal{Z}=\mathcal{S}\times_1\mathbf{Q}\times_2\mathbf{Q}\times_3\mathbf{Q}\times_4\mathbf{Q}$. Then there exists $\underline{\smash{\mathcal{Z}}}=\underline{\smash{\mathcal{S}}}\times_1\underline{\smash{\Q}}\times_2\underline{\smash{\Q}}\times_3\underline{\smash{\Q}}\times_4\underline{\smash{\Q}}\in \underline{\smash{\Bar{\Bar{\mathfrak{S}}}}}_{\mathcal{H}}$ with $\underline{\smash{\mathcal{S}}}\in \underline{\smash{\bar{\bar{\mathfrak{D}}}}}$ and $\underline{\smash{\Q}}\in \underline{\smash{O}}(n)$ obeying $\lVert\mathcal{S}-\underline{\smash{\mathcal{S}}}\rVert_F\leq \epsilon/2$ and $\lVert\Q-\underline{\smash{\Q}}\rVert_{1,2}\leq\epsilon/2$. This gives

\begin{align*}
 \lVert \mathcal{Z}-\underline{\smash{\mathcal{Z}}}\rVert_F&=  \lVert \mathcal{S}\times_1\mathbf{Q}\times_2\mathbf{Q}\times_3\mathbf{Q}\times_4\mathbf{Q}- \underline{\smash{\mathcal{S}}}\times_1\underline{\smash{\mathbf{Q}}}\times_2\underline{\smash{\mathbf{Q}}}\times_3\underline{\smash{\mathbf{Q}}}\times_4\underline{\smash{\mathbf{Q}}}\rVert_F\\
 &=  \lVert \mathcal{S}\times_1\mathbf{Q}\times_2\mathbf{Q}\times_3\mathbf{Q}\times_4\mathbf{Q}+ (\mathcal{S}\times_1\underline{\smash{\mathbf{Q}}}\times_2\underline{\smash{\mathbf{Q}}}\times_3\underline{\smash{\mathbf{Q}}}\times_4\underline{\smash{\mathbf{Q}}}-\mathcal{S}\times_1\underline{\smash{\mathbf{Q}}}\times_2\underline{\smash{\mathbf{Q}}}\times_3\underline{\smash{\mathbf{Q}}}\times_4\underline{\smash{\mathbf{Q}}})\\
 &\qquad\qquad\qquad\qquad\qquad\qquad- \underline{\smash{\mathcal{S}}}\times_1\underline{\smash{\mathbf{Q}}}\times_2\underline{\smash{\mathbf{Q}}}\times_3\underline{\smash{\mathbf{Q}}}\times_4\underline{\smash{\mathbf{Q}}}\rVert_F\\
 &= \lVert \mathcal{S}\times_1(\mathbf{Q}-\underline{\smash{\mathbf{Q}}})\times_2(\mathbf{Q}-\underline{\smash{\mathbf{Q}}})\times_3(\mathbf{Q}-\underline{\smash{\mathbf{Q}}})\times_4(\mathbf{Q}-\underline{\smash{\mathbf{Q}}})+(\mathcal{S}- \underline{\smash{\mathcal{S}}})\times_1\underline{\smash{\mathbf{Q}}}\times_2\underline{\smash{\mathbf{Q}}}\times_3\underline{\smash{\mathbf{Q}}}\times_4\underline{\smash{\mathbf{Q}}}\rVert_F\\
 &\leq\lVert \mathcal{S}\times_1(\mathbf{Q}-\underline{\smash{\mathbf{Q}}})\times_2(\mathbf{Q}-\underline{\smash{\mathbf{Q}}})\times_3(\mathbf{Q}-\underline{\smash{\mathbf{Q}}})\times_4(\mathbf{Q}-\underline{\smash{\mathbf{Q}}})\rVert_F+\lVert(\mathcal{S}- \underline{\smash{\mathcal{S}}})\times_1\underline{\smash{\mathbf{Q}}}\times_2\underline{\smash{\mathbf{Q}}}\times_3\underline{\smash{\mathbf{Q}}}\times_4\underline{\smash{\mathbf{Q}}}\rVert_F
\end{align*}
The first part of the last line gives
\begin{align*}
\lVert \mathcal{S}\times_1(\mathbf{Q}-\underline{\smash{\mathbf{Q}}})\times_2\dots\times_4(\mathbf{Q}-\underline{\smash{\mathbf{Q}}})\rVert_F&=\lVert\text{vec}\left ( \mathcal{S}\times_1(\mathbf{Q}-\underline{\smash{\mathbf{Q}}})\times_2(\mathbf{Q}-\underline{\smash{\mathbf{Q}}})\times_3(\mathbf{Q}-\underline{\smash{\mathbf{Q}}})\times_4(\mathbf{Q}-\underline{\smash{\mathbf{Q}}}) \right )\rVert_2\\
&=\lVert (\mathbf{Q}-\underline{\smash{\mathbf{Q}}})\otimes(\mathbf{Q}-\underline{\smash{\mathbf{Q}}})\otimes (\mathbf{Q}-\underline{\smash{\mathbf{Q}}})\otimes (\mathbf{Q}-\underline{\smash{\mathbf{Q}}})\text{vec}(\mathcal{S})  \rVert_2\\
&\leq \lVert(\mathbf{Q}-\underline{\smash{\mathbf{Q}}})\otimes (\mathbf{Q}-\underline{\smash{\mathbf{Q}}})\otimes (\mathbf{Q}-\underline{\smash{\mathbf{Q}}})\otimes (\mathbf{Q}-\underline{\smash{\mathbf{Q}}}) \rVert_2\lVert\mathcal{S}\rVert_F\\
&= \lVert (\mathbf{Q}-\underline{\smash{\mathbf{Q}}})\rVert^4_2\\
&\leq \lVert (\mathbf{Q}-\underline{\smash{\mathbf{Q}}})\rVert^4_{1,2}\\
&\leq(\epsilon/2)^4\\
&\leq \epsilon/2
\end{align*}
\noindent From line 1 to 2, the identity on pages [477-478] of \cite{TamaraG.Kolda2009} was used. From line 2 to 3 we have used the Cauchy-Schwarz inequality, from line 3 to 4 we have used the equality $\lVert \mathbf{A}\otimes\mathbf{B}\rVert=\lVert\mathbf{A}\rVert\lVert\mathbf{B}\rVert$ and from line 4 to 5 we have used the identity in \cite{rauhut2017low}.
Finally, notice that as $\mathbf{Q}$ is orthogonal
\begin{equation*}
 \lVert(\mathcal{S}- \underline{\smash{\mathcal{S}}})\times_1\underline{\smash{\mathbf{Q}}}\times_2\underline{\smash{\mathbf{Q}}}\times_3\underline{\smash{\mathbf{Q}}}\times_4\underline{\smash{\mathbf{Q}}}\rVert_F=\lVert(\mathcal{S}- \underline{\smash{\mathcal{S}}})\rVert_F=\epsilon/2.
\end{equation*}
Therefore 
\begin{equation*}
    \lVert \mathcal{Z}-\underline{\smash{\mathcal{Z}}}\rVert_F\leq \epsilon/2+\epsilon/2=\epsilon
\end{equation*}
\end{proof}

\par Continuing, we let $\Omega\coloneqq O(n)\times\mathfrak{D}$ define the product set between the set of $n\times n$ orthogonal matrices $O(n)$ and the set of super symmetric cumulant tensors defined in (\ref{eqn: set of diagonal tensors}) and define the map $f:\Omega\mapsto\mathfrak{S}_\mathcal{H}$ by
\begin{equation}
\label{Eqn: model set map}
    f(u)=\mathcal{S}\times_1\mathbf{Q}\times_2\mathbf{Q}\times_3\mathbf{Q}\times_4\mathbf{Q},
\end{equation}
for all $u\coloneqq\left(\mathbf{Q},\mathcal{S}\right)\in\Omega$. Let $\mathcal{Z}=f(u)$ be the tensor corresponding to the image of the map $f$. It is insightful to decompose the normalised secant set $\mathfrak{N}\left(\mathfrak{S}_\mathcal{H}-\mathfrak{S}_\mathcal{H}\right)$ into the set of long and short secants parametrised by some distance $\eta$ \cite{clarkson2008tighter}. The set of long secants of $\mathfrak{S}_\mathcal{H}$ is defined as 
\begin{equation}
\label{Eqn: Set of Long secants}
    \mathfrak{N}_\eta\left(\mathfrak{S}_\mathcal{H}-\mathfrak{S}_\mathcal{H}\right)\coloneqq\left\{\frac{\mathcal{Z}_1-\mathcal{Z}_2}{\left\Vert\mathcal{Z}_1-\mathcal{Z}_2\right\Vert_F}\Bigm| \mathcal{Z}_1,\mathcal{Z}_2\in\mathfrak{S}_\mathcal{H},\,\left\Vert\mathcal{Z}_1-\mathcal{Z}_2\right\Vert_F>\eta \right\}.
\end{equation}
\noindent Furthermore, the set of short secants $\mathfrak{N}_\eta^c\left(\mathfrak{S}_\mathcal{H}-\mathfrak{S}_\mathcal{H}\right)=\mathfrak{N}\left(\mathfrak{S}_\mathcal{H}-\mathfrak{S}_\mathcal{H}\right)\setminus \mathfrak{N}_\eta\left(\mathfrak{S}_\mathcal{H}-\mathfrak{S}_\mathcal{H}\right)$ is the complement to the set of long secants defined by 
\begin{equation}
\label{Eqn: Set of Short secants}
    \mathfrak{N}^c_\eta\left(\mathfrak{S}_\mathcal{H}-\mathfrak{S}_\mathcal{H}\right)\coloneqq\left\{\frac{\mathcal{Z}_1-\mathcal{Z}_2}{\left\Vert\mathcal{Z}_1-\mathcal{Z}_2\right\Vert_F}\Bigm| \mathcal{Z}_1\neq\mathcal{Z}_2\in\mathfrak{S}_\mathcal{H},\,\left\Vert\mathcal{Z}_1-\mathcal{Z}_2\right\Vert_F\leq\eta \right\}.
\end{equation}

\begin{remark}
As the model set $\mathfrak{S}_\mathcal{H}$ is conic (see supplementary material), it is sufficient to cover the normalised secant set of $\Bar{\mathfrak{S}}_\mathcal{H}\coloneqq\mathfrak{S}_\mathcal{H}\cap \mathfrak{B}_1(0)$, where $\mathfrak{B}_1(0)$ denotes the unit Frobenius ball centred at 0, since we have $ \mathfrak{N}\left(\mathfrak{S}_\mathcal{H}-\mathfrak{S}_\mathcal{H}\right)= \mathfrak{N}\left(\Bar{\mathfrak{S}}_\mathcal{H}-\Bar{\mathfrak{S}}_\mathcal{H}\right)$. 
\end{remark}

\par As a result we can decompose the normalised secant set as follows
\begin{align}
\label{Eqn: Normlised Secant Set Decomposition}
\begin{split}
\mathfrak{N}\left(\mathfrak{S}_\mathcal{H}-\mathfrak{S}_\mathcal{H}\right)&= \mathfrak{N}\left(\Bar{\mathfrak{S}}_\mathcal{H}-\Bar{\mathfrak{S}}_\mathcal{H}\right)\\
& = \mathfrak{N}_\eta\left(\Bar{\mathfrak{S}}_\mathcal{H}-\Bar{\mathfrak{S}}_\mathcal{H}\right) \cup \mathfrak{N}_\eta^c\left(\Bar{\mathfrak{S}}_\mathcal{H}-\Bar{\mathfrak{S}}_\mathcal{H}\right)\\ 
& \subseteq \mathfrak{N}_\eta\left(\Bar{\mathfrak{S}}_\mathcal{H}-\Bar{\mathfrak{S}}_\mathcal{H}\right) \cup \mathfrak{N}_\eta^c\left(\Bar{\mathfrak{S}}_\mathcal{H}-\Bar{\Bar{\mathfrak{S}}}_\mathcal{H}\right),\\ 
\end{split}
\end{align}
We begin by covering the set of long secants  $\mathfrak{N}_\eta\left(\Bar{\mathfrak{S}}_\mathcal{H}-\Bar{\mathfrak{S}}_\mathcal{H}\right)$.

\begin{lemma}[Long Secants Covering Number]
Let $\underline{\smash{{\bar{\mathfrak{S}}}}}_\mathcal{H}$ be an $\epsilon\gamma$-cover for $\Bar{\mathfrak{S}}_\mathcal{H}$. Then $\mathfrak{N}\left(\underline{\smash{\bar{\mathfrak{S}}}}_\mathcal{H}-\underline{\smash{\bar{\mathfrak{S}}}}_\mathcal{H}\right)$ is an $\epsilon$-cover for $\mathfrak{N}_{4\gamma}\left(\Bar{\mathfrak{S}}_\mathcal{H}-\Bar{\mathfrak{S}}_\mathcal{H}\right)$ with associated covering number of 
\begin{equation}
    \text{CN}\left(\mathfrak{N}_{4\gamma}\left(\Bar{\mathfrak{S}}_\mathcal{H}-\Bar{\mathfrak{S}}_\mathcal{H}\right),\epsilon\gamma\right)\leq\left(\frac{6}{\epsilon\gamma}\right)^{2n(n+1)}.
\end{equation}
\end{lemma}
\begin{proof}
Lemma 4.1 in \cite{clarkson2008tighter} states that if $\underline{\smash{\bar{\mathfrak{S}}}}_\mathcal{H}$ is a generalised $\epsilon\gamma$- cover of $\Bar{\mathfrak{S}}_\mathcal{H}$, then $\mathfrak{N}\left(\underline{\smash{\bar{\mathfrak{S}}}}_\mathcal{H}-\underline{\smash{\bar{\mathfrak{S}}}}_\mathcal{H}\right)$ is a generalised $\epsilon$-cover for $\mathfrak{N}_{4\gamma}\left(\Bar{\mathfrak{S}}_\mathcal{H}-\Bar{\mathfrak{S}}_\mathcal{H}\right)$. Using the covering number of $\bar{\bar{\mathfrak{S}}}_\mathcal{H}$ from Lemma \ref{App: covering number model set lemma} we get the result.
\end{proof}

\par Continuing, we cover the set of short secants. We begin by stating some preliminary lemmas.
\begin{lemma}[Taylor Approximation Error]
\label{App: Taylor approx lemma}
Let $f:\Omega\mapsto \mathfrak{S}_\mathcal{H}$ be defined as in (\ref{Eqn: model set map}) and let $Df_{u}$ define the first order differential of $f$ evaluated at the point $u$. Further assume that $\lVert\mathcal{S}\rVert_F\leq R$. Then $\forall u,u'\in\Omega$, $\lVert u-u'\rVert\leq2\epsilon_0$, we have
\begin{equation}
\label{Eqn: Taylor Lemma}
    \left\Vert f(u)-f(u')-Df_{u'}^T(u-u')\right\Vert_F\leq C_1\left\Vert u-u'\right\Vert_2^2,
\end{equation}
where $C_1=n^2(n+1)^2\max\left\{3R,1\right\}$
\end{lemma}
\begin{proof}
 w.l.o.g consider the vectorized function $\tilde{f}(u)\coloneqq \vc(f(u))$ such that
\begin{align*}
    \tilde{f}(u) = & \text{vec}\left(\mathcal{S}\times_1\mathbf{Q}\times_2\mathbf{Q}\times_3\mathbf{Q}\times_4\mathbf{Q}\right)\\
    =& \mathbf{Q}\otimes\mathbf{Q}\otimes\mathbf{Q}\otimes\mathbf{Q}\text{ vec}\left(\mathcal{S}\right).
\end{align*}
Using Taylor's theorem \cite[p.~110]{Coleman_2012} of $\tilde{f}$ evaluated at the point $u'\in\Omega$, we get 
\begin{align}
\begin{split}
    \left\Vert \tilde{f}(u)-\tilde{f}(u')-D\tilde{f}_{u'}^T(u-u')\right\Vert_2&\leq \frac{1}{2}\left\Vert (u-u')^TH\tilde{f}_{\xi}(u-u') \right\Vert_2.
\end{split} 
\end{align}
\noindent where $D\tilde{f}_{u}$ and $H\tilde{f}_{u}$ denote the Jacobian and Hessian of $\tilde{f}$ evaluated at $u$ and $\xi=\lambda u +(1-\lambda)u'\in\Omega$, for $\lambda\in(0,1)$, denotes a point on the line segment between $u$ and $u'$. For shorthand let $h=u-u'$, and denote the integer $T\coloneqq \frac{n(n+1)}{2}$, we then have 
\begin{align*}
 \left\Vert  h^T H\tilde{f}_{\xi} h \right\Vert_2 & = \left\Vert\sum_{i=1}^T\sum_{j=1}^T h_i h_j \dfrac{\partial^2\tilde{f}}{\partial u_i\partial u_j}(\xi) \right\Vert_2\\
 & \leq T^2 \max_{i,j}\left\Vert  h_i h_j \dfrac{\partial^2\tilde{f}}{\partial u_i\partial u_j}(\xi) \right\Vert_2\\
 & \leq T^2\left(\max_i\left\vert h_i\right\vert\right)^2 \max_{i,j}\left\Vert\dfrac{\partial^2\tilde{f}}{\partial u_i\partial u_j}(\xi) \right\Vert_2\\
 & = T^2 \lVert h \rVert^2_\infty  \max_{i,j}\left\Vert\dfrac{\partial^2\tilde{f}}{\partial u_i\partial u_j}(\xi) \right\Vert_2\\
  & \leq T^2 \lVert h \rVert^2_2  \max_{i,j}\left\Vert\dfrac{\partial^2\tilde{f}}{\partial u_i\partial u_j}(\xi) \right\Vert_2,\\
\end{align*}
where $h_i=(u_i-u'_i)$. 
w.l.o.g let $\xi=\left(\Q,\mathcal{S}\right)$, we have that
\begin{equation*}
  \max_{i,j}\left\Vert\dfrac{\partial^2\tilde{f}}{\partial u_i\partial u_j}(\xi) \right\Vert_2 = \max\left\{\overset{\circled{1}}{\max_{i,j,k,\ell}\left\Vert\dfrac{\partial^2\tilde{f}}{\partial \mathbf{Q}_{ij}\partial \mathbf{Q}_{kl}}(\xi)\right\Vert_2},     \overset{\circled{2}}{\max_{i,j,k}\left\Vert  \dfrac{\partial^2\tilde{f}}{\partial \mathbf{Q}_{ij}\partial \mathcal{S}_{kkkk}}(\xi) \right\Vert_2},   \overset{\circled{3}}{\max_{i,j}\left\Vert   \dfrac{\partial^2\tilde{f}}{\partial \mathcal{S}_{iiii}\partial \mathcal{S}_{jjjj}}(\xi)\right\Vert_2}\right\}   
\end{equation*}

\begin{enumerate}[label=\protect\circled{\arabic*}]
\item It can be seen that (see supplementary material) 
\begin{equation}
    \dfrac{\partial^2\tilde{f}}{\partial \Q_{ij}\partial \Q_{k\ell}}(\xi) = \Pi_{ijk\ell}\text{vec}\left(\mathcal{S}\right),
\end{equation}
where 
\begin{align*}
{\Pi}_{ijk\ell}
 =&\;\; {\mathbf{E}^{ij}}\otimes \mathbf{E}^{k\ell}\otimes \Q\otimes \Q \;+\; 
 {\mathbf{E}^{ij}}\otimes \Q\otimes \mathbf{E}^{k\ell}\otimes \Q\;+\;
 {\mathbf{E}^{ij}}\otimes \Q\otimes \Q\otimes \mathbf{E}^{k\ell} \\
 & \quad\;+\; \mathbf{E}^{k\ell}\otimes {\mathbf{E}^{ij}}\otimes \Q\otimes \Q\;+\;
 \Q\otimes {\mathbf{E}^{ij}}\otimes \mathbf{E}^{k\ell}\otimes \Q\;+\;
 \Q\otimes {\mathbf{E}^{ij}}\otimes \Q\otimes \mathbf{E}^{k\ell} \\
& \quad\;+\; \mathbf{E}^{k\ell}\otimes \Q\otimes {\mathbf{E}^{ij}}\otimes \Q
 \;+\; \Q\otimes \mathbf{E}^{k\ell}\otimes {\mathbf{E}^{ij}}\otimes \Q
 \;+\; \Q\otimes \Q\otimes {\mathbf{E}^{ij}}\otimes \mathbf{E}^{k\ell} \\
 &\quad\;+\; \mathbf{E}^{k\ell}\otimes \Q\otimes \Q\otimes {\mathbf{E}^{ij}} 
 \;+\; \Q\otimes \mathbf{E}^{k\ell}\otimes \Q\otimes {\mathbf{E}^{ij}} 
 \;+\; \Q\otimes \Q\otimes \mathbf{E}^{k\ell}\otimes {\mathbf{E}^{ij}} \\
\end{align*}
and the matrix $\mathbf{E}^{ij}=\mathbf{e}_i\mathbf{e}_j^T$, where $\mathbf{e}_i$ is the $i$th unit basis vector. Using the properties of the Kronecker product and the triangle inequality we get
\begin{align*}
   \left\Vert\dfrac{\partial^2\tilde{f}}{\partial \mathbf{Q}_{ij}\partial \mathbf{Q}_{kl}}(\xi)\right\Vert_2 & \leq 12 \left\Vert \mathbf{E}^{ij}\right\Vert_2 \left\Vert \mathbf{E}^{kl}\right\Vert_2 \left\Vert \Q \right\Vert^2_2\left\Vert \mathcal{S}\right\Vert_F\\
   & = 12 \left\Vert \mathcal{S}\right\Vert_F. 
\end{align*}
Assuming that the diagonal tensor has bounded support $\lVert \mathcal{S}\rVert_2\leq R$, then it follows that  
\begin{equation}
    \max_{i,j,k,\ell}\left\Vert\dfrac{\partial^2\tilde{f}}{\partial \mathbf{Q}_{ij}\partial \mathbf{Q}_{kl}}(\xi)\right\Vert_2\leq 12 R.
\end{equation}
\item It can be seen that (see supplementary material) 
\begin{equation}
    \dfrac{\partial^2\tilde{f}}{\partial \Q_{ij}\partial \mathcal{S}_{kkkk}}(\xi) = \Gamma_{ij}\mathbf{e}_k,
\end{equation}
where and
\begin{align*}
    \Gamma_{ij}&= \mathbf{E}^{ij}\otimes \Q\otimes \Q\otimes \Q
\;+\; \mathbf{Q}\otimes \mathbf{E}^{ij}\otimes \Q\otimes \Q\\
&\quad+\; \Q\otimes \Q\otimes \mathbf{E}^{ij}\otimes \Q
\;+\;\Q\otimes \Q\otimes \Q\otimes \mathbf{E}^{ij}.\\
\end{align*}
Similarly to $\circled{1}$, we get
\begin{equation*}
 \max_{i,j,k}\left\Vert  \dfrac{\partial^2\tilde{f}}{\partial \mathbf{Q}_{ij}\partial \mathcal{S}_{kkkk}}(\xi) \right\Vert_2  \leq 4 
\end{equation*}
\item It can be easily shown that 
\begin{equation}
    \dfrac{\partial^2\tilde{f}}{\partial \mathcal{S}_{iiii}\partial \mathcal{S}_{jjjj}}(\xi)=\mathbf{0},
\end{equation}
therefore 
\begin{equation}
\max_{i,j}\left\Vert   \dfrac{\partial^2\tilde{f}}{\partial \mathcal{S}_{iiii}\partial \mathcal{S}_{jjjj}}(\xi)\right\Vert_2 = 0.    
\end{equation}
\end{enumerate}
It therefore follows that 
\begin{equation}
\max_{i,j}\left\Vert\dfrac{\partial^2\tilde{f}}{\partial u_i\partial u_j}(\xi) \right\Vert_2 = \max \left\{12R,4\right\},    
\end{equation}
and, 
\begin{equation}
\left\Vert \tilde{f}(u)-\tilde{f}(u')-D\tilde{f}_{u'}^T(u-u')\right\Vert_2\leq n^2(n+1)^2\max\left\{3R,1\right\}\left\Vert u-u'\right\Vert^2_2.
\end{equation}
\end{proof}

\begin{lemma}[Bounded Curvature]
\label{App: bounded curvature lemma}
Let $f:\Omega\mapsto \mathfrak{S}_\mathcal{H}$ be defined as in (\ref{Eqn: model set map}) and let $Df_{u}$ define the first order differential of $f$ evaluated at the point $u$. Further assume that $\lVert\mathcal{S}\rVert_F\leq R$. Then $\forall u,u'\in\Omega$, $\lVert u-u'\rVert\leq2\epsilon_0$, we have
\begin{equation}
\label{Eqn: curv lemma}
    \left\Vert Df_{u}-Df_{u'}\right\Vert_F\leq C_2\left\Vert u-u'\right\Vert_2,
\end{equation}
where $C_1=2C_2$
\end{lemma}
\begin{proof}
Using the mean value theorem \cite{Coleman_2012}, it can be shown that,
\begin{align}
    \left\Vert  D\tilde{f}_u \;-\; D\tilde{f}_{u'}\right\Vert_2 & \leq \left\Vert H\tilde{f}_{\xi}^T(u-u')\right\Vert_2
\end{align}
 for some $\xi=\lambda u +(1-\lambda)u'\in\Omega$, for $\lambda\in(0,1)$. Then using the same argument as in the proof of Lemma \ref{App: Taylor approx lemma}, it can easily shown that
 \begin{equation}
       \left\Vert  D\tilde{f}_u \;-\; D\tilde{f}_{u'}\right\Vert_2 \leq 2C_1\left\Vert u-u'\right\Vert_2,
 \end{equation}
giving $C_2=2C_1$.
\end{proof}

\begin{lemma}[Bounded Gradient]
\label{App: bounded gradient lemma}
Let $f:\Omega\mapsto \mathfrak{S}_\mathcal{H}$ be defined as in (\ref{Eqn: model set map}) and let $Df_{u}$ define the first order differential of $f$ evaluated at the point $u$. Further assume, as in (\ref{eqn: set of diagonal tensors}), that $\mathcal{S}_{iiii}\geq\epsilon_\mathcal{S} (> 0)$ $\forall i$.  Then $\forall u\in\Omega$
\begin{equation}
\label{Eqn: Bounded Gradient}
    \left\Vert Df_{u}^\dagger\right\Vert_F\leq C_3,
\end{equation}
where $C_3=2\epsilon_\mathcal{S}$
\end{lemma}
\begin{proof}
As in \ref{App: Taylor approx lemma}, we consider the vectorized function $\tilde{f}(u)\coloneqq \vc\left(f(u)\right)$ w.l.o.g. It can be seen that the 1st order differential (see supplementary material) has the following decomposition
\begin{equation}
    D\tilde{f}(u)= 
    \begin{bmatrix}
    \dfrac{\partial \tilde{f}}{\partial\Q}(u), \dfrac{\partial\tilde{f}}{\partial\mathcal{S}}(u)
    \end{bmatrix},
\end{equation}
where 
\begin{equation}
\label{eqn: partial diff Qij}
 \dfrac{\partial \tilde{f}}{\partial\Q_{ij}}(u)\; = \; \Gamma_{ij} \vc\left(\mathcal{S}\right). 
\end{equation}
Furthermore, the partial derivative with respect to the super symmetric cumulant tensor $\mathcal{S}$ is defined as 
\begin{align*}
 \dfrac{\partial\tilde{f}}{\partial\mathcal{S}}(u) \; = \; \mathbf{B}. 
\end{align*}
where $\mathbf{B}\coloneqq \Q\otimes\Q\otimes\Q\otimes \Q$. Equivalently, (\ref{Eqn: Bounded Gradient}) can be rewritten as 
\begin{equation}
    \label{eqn: rewrite assumption 3}
    \min_{\lVert\Delta u\rVert=1}\left\Vert D\tilde{f}\left(u\right)^T\Delta u\right\Vert_2 \geq C_3,
\end{equation}
where $\Delta u=\left(\Delta\mathbf{Q},\Delta\mathcal{S}\right)$. We therefore have 
\begin{align*}
    \left\Vert D\tilde{f}(u)^T \Delta u \right\Vert^2_2 & = \left\Vert\dfrac{\partial \tilde{f}}{\partial\mathbf{Q}}(u)^T \Delta\mathbf{Q} \right\Vert^2_F + \left\Vert \dfrac{\partial \tilde{f}}{\partial\mathcal{S}}(u)^T \Delta\mathcal{S} \right\Vert^2_2\\
    & = \sum_{i=1}^n\sum_{j=1}^n\left\Vert\dfrac{\partial \tilde{f}}{\partial\mathbf{Q}_{ij}}(u)^T \Delta\mathbf{Q}_{ij}\right\Vert^2_F +  \left\Vert \dfrac{\partial \tilde{f}}{\partial\mathcal{S}}(u)^T \Delta\mathcal{S} \right\Vert^2_2\\
    & = (\star).
\end{align*}
As $f$ is equivariant in $\mathbf{Q}$, we can set $\mathbf{Q}=\mathbf{I}_n$ w.l.o.g. As a result $\mathbf{B}=\mathbf{I}$ and $\Gamma_{ij}$ reduces to 

\begin{align*}
    \Gamma_{ij} \; & = \; \mathbf{E}^{ij} \otimes \mathbf{I}_n \otimes \mathbf{I}_n \otimes \mathbf{I}_n \; + \; \mathbf{I}_n\otimes \mathbf{E}^{ij} \otimes \mathbf{I}_n \otimes \mathbf{I}_n \\
    \; & \:\qquad +  \mathbf{I}_n \otimes \mathbf{I}_n \otimes \mathbf{E}^{ij} \otimes \mathbf{I}_n \; + \; \mathbf{I}_n\otimes \mathbf{I}_n \otimes \mathbf{I}_n\otimes \mathbf{E}^{ij}. 
\end{align*}

\noindent For shorthand, let $\mathcal{T}=\Gamma_{ab}\vc\left(\mathcal{S}\right)$ and noting that $\mathbf{E}^{ab}=\mathbf{e}_a\mathbf{e}_b^T$, we have

\begin{align*}
  \mathcal{T}_{ijk\ell}&=\sum_{p=1}^n \left(\mathbf{E}^{ab}_{ip}\mathbf{I}_{jp}\mathbf{I}_{kp}\mathbf{I}_{\ell p} + \mathbf{I}_{ip}\mathbf{E}^{ab}_{jp}\mathbf{I}_{kp}\mathbf{I}_{\ell p} + \mathbf{I}_{ip}\mathbf{I}_{jp}\mathbf{E}^{ab}_{kp}\mathbf{I}_{\ell p} +  \mathbf{I}_{ip}\mathbf{I}_{jp}\mathbf{I}_{kp} \mathbf{E}^{ab}_{\ell p}\right)\mathcal{S}_{pppp}\\
  &= \sum_{p=1}^n \left(\delta_{ai}\delta_{bp}\delta_{jp}\delta_{kp}\delta_{\ell p} + \delta_{ip}\delta_{aj}\delta_{bp}\delta_{kp}\delta_{\ell p} + \delta_{ip}\delta_{jp}\delta_{ak}\delta_{bp}\delta_{\ell p} +  \delta_{ip}\delta_{jp}\delta_{kp} \delta_{a\ell}\delta_{bp}\right)\mathcal{S}_{pppp}\\
  &= \sum_{p=1}^n \left(\delta_{ai}\delta_{jp}\delta_{kp}\delta_{\ell p} + \delta_{ip}\delta_{aj}\delta_{kp}\delta_{\ell p} + \delta_{ip}\delta_{jp}\delta_{ak}\delta_{\ell p} +  \delta_{ip}\delta_{jp}\delta_{kp} \delta_{a\ell}\right)\delta_{bp}\mathcal{S}_{pppp}\\
   &= \left(\delta_{ai}\delta_{jb}\delta_{kb}\delta_{\ell b} + \delta_{ib}\delta_{aj}\delta_{kb}\delta_{\ell b} + \delta_{ib}\delta_{jb}\delta_{ak}\delta_{\ell b} +  \delta_{ib}\delta_{jb}\delta_{kb} \delta_{a\ell}\right)\mathcal{S}_{bbbb}.
\end{align*}

As a result, we have that

\begin{align*}
    \left\Vert \Gamma^{ab}\vc\left(\mathcal{S}\right)\Delta\mathbf{Q}_{ab} \right\Vert_{F}^2 & =  \sum_{i,j,k,\ell=1}^n \left\vert\left(\delta_{ai}\delta_{jb}\delta_{kb}\delta_{\ell b} + \delta_{ib}\delta_{aj}\delta_{kb}\delta_{\ell b} + \delta_{ib}\delta_{jb}\delta_{ak}\delta_{\ell b} +  \delta_{ib}\delta_{jb}\delta_{kb} \delta_{a\ell}\right)\mathcal{S}_{bbbb}\Delta\mathbf{Q}_{ab}\right\vert^2.
\end{align*}
It can be easily shown that for $a=b$

\begin{align*}
    \left\Vert \Gamma^{bb}\vc\left(\mathcal{S}\right)\Delta\mathbf{Q}_{bb} \right\Vert_{F}^2 & = 16\left\vert \mathcal{S}_{bbbb}\Delta\mathbf{Q}_{bb} \right\vert^2,
\end{align*}
and for $a\neq b$

\begin{align*}
    \left\Vert \Gamma^{ab}\vc\left(\mathcal{S}\right)\Delta\mathbf{Q}_{ab} \right\Vert_{F}^2 & = 4\left\vert \mathcal{S}_{bbbb}\Delta\mathbf{Q}_{ab} \right\vert^2.
\end{align*}

\noindent We therefore have
\begin{align*}
    \left(\star\right) & = \sum_{i=j} \left\Vert\dfrac{\partial \tilde{f}}{\partial\mathbf{Q}_{ii}}(u)^T \Delta\mathbf{Q}_{ii}\right\Vert^2_F + \sum_{i\neq j} \left\Vert\dfrac{\partial \tilde{f}}{\partial\mathbf{Q}_{ij}}(u)^T \Delta\mathbf{Q}_{ij}\right\Vert^2_F +  \left\Vert  \Delta\mathcal{S} \right\Vert_2^2\\
    & = 16\sum_{i=j}\left\vert\mathcal{S}_{iiii}\right\vert^2\left\vert\Delta\mathbf{Q}_{ii}\right\vert^2  + 4\sum_{i\neq j}\left\vert\mathcal{S}_{iiii}\right\vert^2\left\vert\Delta\mathbf{Q}_{ij}\right\vert^2 + \left\Vert\Delta\mathcal{S}\right\Vert^2_2\\
     & \geq 4\sum_{i,j }\left\vert\mathcal{S}_{iiii}\right\vert^2\left\vert\Delta\mathbf{Q}_{ij}\right\vert^2 + \left\Vert\Delta\mathcal{S}\right\Vert^2_2\\
     & = \left(\star\right)
\end{align*}
Now assume that $\left\vert\mathcal{S}_{iiii}\right\vert\geq\epsilon_{\mathcal{S}}$ for all $i$, therefore
\begin{align*}
    \left(\star\right) & \geq 4\epsilon^2_{\mathcal{S}}\left\Vert\Delta\mathbf{Q}\right\Vert^2_F + \left\Vert\Delta\mathcal{S}\right\Vert^2_2\\
    & \geq 4\epsilon_{\mathcal{S}}^2\left\Vert\Delta u\right\Vert_2^2.
\end{align*}
In the last line, we assume w.l.o.g that $4\epsilon^2_\mathcal{S}\leq 1$. We have therefore proved that 
\begin{equation}
  \min_{\lVert\Delta u\rVert=1}\left\Vert D\tilde{f}\left(u\right)^T\Delta u\right\Vert_2 \geq 2\epsilon_\mathcal{S},  
\end{equation}
yielding $C_3\coloneqq2\epsilon_\mathcal{S}$.
\end{proof}

\par We have the following lemma to cover the set of short secants. 
\begin{lemma}[Short Secants Covering Number]
\label{App: short secants covering number lemma}
Let $\underline{\smash{\Omega'}}=\left\{u_i\right\}$ be an $\epsilon$- cover for $\Omega' = O(n)\times\left(\mathfrak{D}\cap\mathfrak{B}_1(0)\right)$ and considering the following:
\begin{enumerate}
    \item $\left\Vert f(u)-f(u')-Df_{u'}^T(u-u')\right\Vert\leq C_1\left\Vert u-u'\right\Vert^2\qquad$ (Taylor approximation Lemma \ref{App: Taylor approx lemma})
    \item $\left\Vert Df_{u}-Df_{u'}\right\Vert\leq C_2\left\Vert u-u'\right\Vert\qquad$ (bounded curvature Lemma \ref{App: bounded curvature lemma})  
    \item $\left\Vert Df_u^\dagger\right\Vert\leq C_3\qquad$ (bounded gradient Lemma \ref{App: bounded gradient lemma}),
\end{enumerate}
where $f:\Omega\mapsto \mathfrak{S}_\mathcal{H}$ is defined in Eqn. \ref{Eqn: model set map} and $Df_u$ defines the first order differential of $f$ evaluated at the point $u$. Then given $u_i\in\Omega$, $\forall u,u'\in\mathfrak{B}_{\epsilon_0}(u_i)$ and $\left\Vert\mathcal{Z}-\mathcal{Z}'\right\Vert\leq\eta$, where $\mathcal{Z}=f(u)$ and $\mathcal{Z}'=f(u')$, we have
\begin{equation}
    \left\Vert \frac{\mathcal{Z}-\mathcal{Z}'}{\left\Vert\mathcal{Z}-\mathcal{Z}'\right\Vert} - Df_{u_i}^T\frac{u-u'}{\left\Vert\mathcal{Z}-\mathcal{Z}'\right\Vert}     \right\Vert \leq C_4\epsilon_0.
\end{equation}
where $C_4\coloneqq C_3(2C_1+C_2)$.
\end{lemma}
\begin{proof}
\begin{align*}
    \left\Vert \mathcal{Z}-\mathcal{Z}'-Df_{u_i}^T(u-u')\right\Vert & = \left\Vert f(u)-f(u')-Df_{u}^T(u-u')+\left(Df_{u}-Df_{u_i}\right)^T\left(u-u'\right)\right\Vert\\
    &\leq \left\Vert f(u)-f(u') - Df_{u}^T(u-u')\right\Vert + \left\Vert \left(Df_{u}-Df_{u_i}\right)^T\left(u-u'\right)\right\Vert\\
    &\leq C_1 \left\Vert u-u' \right\Vert^2 +C_2\left\Vert u-u_i \right\Vert \left\Vert u-u' \right\Vert\\
    & = (\star)
\end{align*}
Given that $u,u'\in \mathfrak{B}_{\epsilon_0}(u_i)$, we have that $\left\Vert u-u_i\right\Vert\leq\epsilon_0$ and $\lVert u-u'\rVert\leq 2\epsilon_0$. Therefore 
\begin{align*}
    (\star) &\leq 2C_1\epsilon_0\left\Vert u-u'\right\Vert + C_2\epsilon_0 \left\Vert u -u'\right\Vert\\
    &= (2C_1+C_2)\epsilon_0\left\Vert u-u'\right\Vert.
\end{align*}
Now dividing by $\left\Vert\mathcal{Z}-\mathcal{Z}'\right\Vert$ gives:
\begin{align*}
  \left\Vert \frac{\mathcal{Z}-\mathcal{Z}'}{\left\Vert\mathcal{Z}-\mathcal{Z}'\right\Vert} - Df_{u_i}^T\frac{u-u'}{\left\Vert\mathcal{Z}-\mathcal{Z}'\right\Vert}     \right\Vert & \leq (2C_1+C_2) \frac{\left\Vert u-u' \right\Vert}{\left\Vert \mathcal{Z}-\mathcal{Z}'\right\Vert}\\
 & \leq  C_3 (2C_1+C_2).
\end{align*}
In the last line, we have used the fact that bounded (inverse) gradient implies Lipschitzness. 
\end{proof}

As a result, the set of bounded tangent vectors, defined by
\begin{equation*}
    \mathcal{V}\coloneqq\left\{Df_{u_i}^T\frac{u-u'}{\left\Vert\mathcal{Z}-\mathcal{Z}'\right\Vert} \mid \forall u_i\in\Omega\right\}
\end{equation*} forms a generalized $\epsilon$-cover for $\mathfrak{N}_\eta^c\left(\Bar{\mathfrak{S}}_\mathcal{H}-\Bar{\Bar{\mathfrak{S}}}_\mathcal{H}\right)$ with covering number (see Lemma 4.3 of \cite{clarkson2008tighter}) 
\begin{align*}
    \text{CN}\left(\mathcal{V},\epsilon \right) & \leq C_4 \text{ CN}\left(\bar{\bar{\mathfrak{S}}}_\mathcal{H},\epsilon_0\right) \left( \frac{3}{\epsilon}\right)^{\frac{n(n+1)}{2}}\\
    & \leq C_4 \left( \frac{6}{\epsilon_0}\right)^{n(n+1)}\left( \frac{3}{\epsilon}\right)^{\frac{n(n+1)}{2}}.
\end{align*}

\par From (\ref{Eqn: Normlised Secant Set Decomposition}), we can bound the covering number of the normalized secant set:
\begin{align*}
    \text{CN}\left(\mathfrak{N}\left(\mathfrak{S}_\mathcal{H}-\mathfrak{S}_\mathcal{H}\right),\epsilon\right) & \leq \text{CN}\left(\mathfrak{N}_\eta\left(\Bar{\mathfrak{S}}_\mathcal{H}-\Bar{\mathfrak{S}}_\mathcal{H}\right),\epsilon\right)\;+\; \text{CN}\left(\mathfrak{N}_\eta^c\left(\Bar{\mathfrak{S}}_\mathcal{H}-\Bar{\Bar{\mathfrak{S}}}_\mathcal{H}\right),\epsilon\right)\\
    & \leq \left( \frac{6}{\gamma\epsilon}\right)^{2n(n+1)} \;+\; C_4 \left( \frac{6}{\epsilon_0}\right)^{n(n+1)}\left( \frac{3}{\epsilon}\right)^{\frac{n(n+1)}{2}}\\
     & \leq \left( \frac{6}{\gamma\epsilon}\right)^{2n(n+1)} \;+\; C_4 \left( \frac{6}{\epsilon_0}\right)^{n(n+1)}\left( \frac{3}{\epsilon}\right)^{n(n+1)}\\
     & =  \left( \frac{6}{\gamma\epsilon}\right)^{2n(n+1)} \;+\; C_4 \left( \frac{18}{\epsilon_0\epsilon}\right)^{n(n+1)}\\
     & = (\star).
\end{align*}
Note that by definition $\epsilon_0\leq\eta  \;(=4\gamma)$, therefore $\gamma\geq\frac{\epsilon_0}{4}$. As a result
\begin{align*}
    (\star) & \leq C_4 \left( \left(\frac{24}{\epsilon_0\epsilon}\right)^{2n(n+1)} \;+\; \left(\frac{24}{\epsilon_0\epsilon}\right)^{n(n+1)} \right)\\
   & \leq 2 C_4\left(\frac{C_5}{\epsilon}\right)^{2n(n+1)}.
\end{align*}

\section{Proof of Theorem \ref{Thm: finite sample effects}}
\label{App: Proof of finite sample effects}
\begin{proof}
Let $\mathcal{Z}$ denote the expected cumulant tensor and $\hat{\mathcal{Z}}=\frac{1}{N}\sum^n_{i=1}\z_i^{\otimes^4}$ the cumulant tensor computed from the finite data samples $\{\z_i\}^N_{i=1}$ as in (\ref{eqn: finite comp of tensor}). Note that the feature function $\Phi^\textbf{w}$ in (\ref{Eqn: CICA Feature Function}) can be equivalently written as $\Phi^\textbf{w}(\z)=\mathcal{A}(\z^{\otimes^4})$. Using the RIP result from Theorem \ref{Thm: RIP Result} and by assuming the independent components have bounded support $\lVert\mathcal{S}\rVert_F\leq R$, we have
\begin{align*}
   \lVert\Phi^\textbf{w}(\z)\rVert&=\lVert\mathcal{A}(\z^{\otimes^4})\rVert_F\\ 
   &\leq \sqrt{1+\delta}\lVert\z^{\otimes^4}\rVert_F\\
   &=\sqrt{1+\delta}\lVert\s^{\otimes^4}\rVert_F\\
    &=\sqrt{1+\delta}\lVert\mathcal{S}\rVert_F\\
   &\leq\sqrt{1+\delta}R.
\end{align*}
where the $l_2$ norm is invariant under the orthogonal transformation $\Q$. Next we apply the concentration of averages lemma (lemma 4) of \cite{rahimi2009weighted} to get with probability $1-\xi$ on the drawing of both $\z_i$ and $\mathcal{A}$ that
\begin{equation*}
    \lVert\mathcal{A}(\mathcal{Z})-\mathcal{A}(\hat{\mathcal{Z}})\rVert_2\leq \frac{R\sqrt{2(1+\delta)\log(1/\xi)}}{\sqrt{N}}
\end{equation*}
\end{proof}

\end{document}